\RequirePackage[svgnames,table]{xcolor}

\documentclass[11pt,letterpaper,logo]{yalearxiv}

\usepackage{graphicx}
\usepackage{float,epstopdf}
\usepackage{bbm}

\usepackage{microtype}

\usepackage{natbib}
\setcitestyle{square}

\usepackage{subcaption}
\usepackage{booktabs}

\usepackage{amsmath}
\usepackage{amssymb}
\usepackage{mathtools}
\usepackage{amsthm}
\usepackage{dsfont}
\usepackage{multicol}
\usepackage{makecell}
\usepackage{multirow} 
\usepackage{amsfonts} 
\usepackage{mathrsfs}
\usepackage[amssymb, thickqspace]{SIunits}
\usepackage{enumitem}
\usepackage{pgfplotstable}
\usepackage{lipsum}		% Can be removed after putting your text content

\usepackage{microtype}
\usepackage{graphicx}
\usepackage{subcaption}
\usepackage{booktabs} % for professional tables
\usepackage[table]{xcolor}
\usepackage{arydshln}
\usepackage[normalem]{ulem} % [normalem] prevents the package from changing the default behavior of `\emph` to underline.

\usepackage{cases}
\usepackage{wrapfig}

\usepackage{url}

\usepackage{thmtools}
\usepackage{thm-restate}
\usepackage{tabu}

\definecolor{huskypurple}{HTML}{4B2E83}
\newcommand{\ourclr}{purple}
\usepackage{titletoc}
\makeatletter
\def\munderbar#1{\underline{\sbox\tw@{$#1$}\dp\tw@\z@\box\tw@}}
\makeatother

\newtheorem{definition}{Definition}[section]
\newtheorem{theorem}[definition]{Theorem}

\newtheorem{corollary}[definition]{Corollary}
\newtheorem{proposition}[definition]{Proposition}

\newtheorem{assumption}[definition]{Assumption}

\AddToHook{cmd/appendix/before}{%
  \setcounter{axiom}{0}%
}

\newcommand{\be}{\begin{equation}}
\newcommand{\ee}{\end{equation}}
\newcommand{\bea}{\begin{equation*}\begin{aligned}}
\newcommand{\eea}{\end{aligned}\end{equation*}}

\newcommand{\mc}{\mathcal}

\newcommand{\ve}[1]{\mathbf{#1}}
\newtcolorbox{simpleElegantQuote}{
    colback=AliceBlue!50!White,   
    colframe=RoyalBlue!75!Black,  
    boxrule=0.5pt,                
    arc=2mm,                     
    boxsep=4pt,
    left=10pt, right=10pt,        
    top=8pt, bottom=8pt,         
    fontupper=\itshape,          
}

\title{Make Each Token Count: Towards Improving Long-Context Performance with KV Cache Eviction}

\runningtitle{Make Each Token Count: Towards Improving Long-Context Performance with KV Cache Eviction}

\usepackage{fontawesome5}
\definecolor{myorange}{RGB}{250,99,57}
\newcommand{\orgfire}{\textcolor{myorange}{\faFire*}}
\newcommand{\trimkv}{TrimKV\textsuperscript{ \orgfire}}
\newcommand{\dbtrimkv}{DBTrimKV\textsuperscript{ \orgfire}}

\definecolor{yaleblue}{RGB}{0,58,112}
\definecolor{jpmcbrown}{RGB}{153,108,72}
\definecolor{cuhkviolet}{RGB}{117,15,109}
\newcommand{\yale}{%
  \textsuperscript{%
    {\usefont{T1}{pbk}{m}{n}\textcolor{yaleblue}{\textbf{Y}}}%
  }%
}
\newcommand{\cuhk}{%
  \textsuperscript{%
    {\usefont{T1}{pbk}{m}{n}\textcolor{cuhkviolet}{\textbf{C}}}%
  }%
}

\author{Ngoc Bui\yale, Hieu Trung Nguyen\cuhk, Arman Cohan\yale, Rex Ying\yale\\
\yale Department of Computer Science, Yale University,
\cuhk The Chinese University of Hong Kong\\
\faGithub~\textbf{Source Code:} \href{https://github.com/ngocbh/trimkv}{\texttt{https://github.com/ngocbh/trimkv}} 
}

\newcommand\DoToC{%
  \startcontents
  \printcontents{}{1}{\textbf{Table of Contents}\vskip3pt\hrule\vskip5pt}
  \vskip3pt\hrule\vskip5pt
}

\begin{document}

\begin{abstract}
\vspace{-1mm}
{\centering\section*{Abstract}}
The key-value (KV) cache is a major bottleneck in long-context inference, where memory and computation grow with sequence length. Existing KV eviction methods reduce this cost but typically degrade performance relative to full-cache inference. Our key insight is that full-cache attention is not always optimal: in long contexts, irrelevant tokens can dilute attention away from useful evidence, so selective, learnable eviction can improve generation rather than merely approximate the full cache. We introduce a global retention-based KV eviction method that learns each token's future utility under a unified memory budget. Lightweight retention gates assign utility scores to cached KV entries, and a shared final scoring projection encourages score comparability across all layers and heads. This enables a single global eviction policy in which tokens from different layers, heads, and modalities compete directly for cache capacity.
We further provide theoretical analysis showing that preferentially retaining useful tokens can reduce attention dilution, and we justify geometric retention as a query-agnostic proxy for future utility. Across long-context language and vision-language benchmarks, our method substantially reduces KV memory while maintaining performance close to full-cache inference. These results suggest that learned, globally ranked KV eviction is not only a compression technique, but also a potential mechanism for improving long-context performance.
\end{abstract}

\maketitle

\section{Introduction}
\label{sec:intro}

The key-value (KV) cache is central to efficient autoregressive decoding in transformer-based language and vision--language models (LLMs and VLMs). By storing past keys and values, the model avoids recomputing representations for previous tokens during generation. However, this cache grows linearly with sequence length, and the attention computation over cached tokens grows with the amount of retained context. This becomes a major bottleneck in long-context and multimodal settings, where prompts may contain tens of thousands of text tokens or hundreds to thousands of visual tokens from images and videos~\citep{huang2025vision, sapkota2025vision, tu2024vl}. As context windows continue to expand, KV cache management has become one of the main challenges for practical long-context inference.

A common solution is KV eviction: once the cache exceeds a memory budget, the system removes tokens estimated to be unimportant. Existing methods typically view eviction as a compression problem, aiming to approximate full-cache inference while reducing memory and computation~\citep{li2024survey}. Many policies rely on heuristics such as recency, accumulated attention, or local attention magnitude~\citep{xiao2023efficient, li2024snapkv, zhang2023h2o, cai2025r}. While effective at reducing cost, these methods often degrade model quality relative to full-cache inference. This degradation is usually treated as unavoidable: removing context is assumed to trade accuracy for efficiency.

We revisit this assumption. Our key insight is that full-cache inference is not always ideal in long contexts. When many irrelevant or weakly relevant tokens remain in the cache, self-attention must normalize over all of them. As a result, useful evidence competes with a growing number of distractors, and attention mass can be diluted away from the tokens needed for prediction~\citep{bansal2026lets}. From this perspective, KV eviction is not merely an approximation to full-cache attention. If the right tokens are removed, eviction can suppress distractors, sharpen attention, and improve generation.

This perspective raises two central questions. First, \emph{how can a model identify which cached tokens will remain useful for future decoding?} Attention-based eviction heuristics are limited because attention scores are query-dependent and often reflect short-term relevance to the current prediction, rather than persistent utility across later decoding steps, subproblems, turns, or modalities. Second, \emph{how should a limited KV budget be allocated across layers and heads?} Different layers and heads may serve different roles: some preserve long-range information, while others mainly attend to local or short-lived context. A fixed per-layer or per-head budget can therefore misallocate memory. Existing methods often treat these two questions separately. Some focus on estimating token importance~\citep{xiao2023efficient, li2024snapkv, cai2025r, bui2025cache}, while others design budget-allocation rules across heads, layers, or modalities~\citep{feng2024ada, qin2025cake, shi2023adapyramid, tu2024vl}, but often relying on myopic attention statistics. 

\textbf{In this work}, we seek a simple unified solution based on learnable retention scores. Building on token retention~\citep{bui2025cache}, we use lightweight retention gates to predict a scalar future-utility score for each cached KV entry. These scores are trained under a memory constraint to capture whether a token is likely to remain useful for future decoding, rather than merely how much it is attended to at the current step. To make scores comparable across layers and heads, we tie the final scoring projection of all retention gates. This weight sharing encourages retention scores to share a common scale.

With globally comparable scores, KV eviction becomes a single ranking problem. Instead of imposing fixed budgets for each layer or head, we maintain one global KV budget and retain the entries with the highest predicted utility across all layers, heads, and modalities. This allows cache capacity to be allocated dynamically: layers and heads that preserve useful long-range information can receive more memory, while those dominated by low-utility or distracting tokens receive less. The resulting policy jointly performs token selection and budget allocation through the same learned retention score.

Empirically, we evaluate our method on long-context language and vision-language benchmarks. Across long-horizon reasoning, multi-turn dialogue, and multimodal understanding tasks, global retention substantially reduces KV memory while maintaining performance close to full-cache inference. In some settings, selective eviction improves accuracy over the full cache, suggesting that removing low-utility tokens may improve task performance rather than simply reduce cost.

Our contributions are summarized as follows:
\begin{itemize}[leftmargin=2em, itemsep=0pt, topsep=0pt, parsep=1pt, partopsep=1pt]
    \item We identify attention dilution as a mechanism by which full-cache inference can degrade in long contexts, and show theoretically that under suitable retention conditions, preferentially evicting distractors can improve attention quality. We justify geometric retention as a query-agnostic surrogate for future token utility, supported by empirical survival patterns of attended tokens in long-contexts.

    \item We introduce weight-tied retention gates that learn future-utility scores and encourage their comparability across layers and heads, enabling direct global comparison of cached KV entries. We propose a global retention-based KV eviction policy that jointly performs token selection and dynamic cache allocation under a single memory budget across layers, heads, and modalities.

    \item Through experiments on long-context language and vision-language benchmarks, we show that our method improves efficiency and can maintain performance close to full-cache inference while using substantially less KV memory.
\end{itemize}

\section{Preliminaries}

\subsection{Self-Attention and KV Eviction}

Consider autoregressive generation in a transformer with self-attention. At decoding step \(t\), the KV cache contains all previously generated tokens $C_t = \{1,\dots,t\}$, where each token \(i \in C_t\) is associated with a key-value pair \((\mathbf{k}_i,\mathbf{v}_i)\). Given the query \(\mathbf{q}_t\) at step \(t\), attention output is 
\[
z_{t,i} := \frac{\mathbf{q}_t^\top \mathbf{k}_i}{\sqrt d},
\qquad
\alpha_{t,i} := \frac{e^{z_{t,i}}}{\sum_{j \in C_t} e^{z_{t,j}}}, \qquad \mathbf{o}_t = \sum_{i \in C_t} \alpha_{t,i}\mathbf{v}_i.
\]

As decoding proceeds, the cache grows linearly with \(t\). Consequently, KV memory scales with context length, while the cumulative attention cost over generation scales quadratically~\citep{keles2023computational}.

A standard approach for improving memory and computation is to restrict the cache to at most \(M\) tokens by evicting less important key-value pairs. Under such an eviction policy, the attention becomes 
\begin{equation}\label{eq:reform}
\alpha^r_{t,i}
=
\frac{r_{t,i} e^{z_{t,i}}}
{\sum_{j \in C_t} r_{t,j} e^{z_{t,j}}},
\qquad
\mathbf{o}'_t = \sum_{i \in C_t} \alpha^r_{t,i}\mathbf{v}_i,
\qquad
r_{t,i}\in\{0,1\},
\qquad
r_{t,i}\ge r_{t+1,i}.
\end{equation}
Here, \(r_{t,i}\) indicates whether token \(i\) is retained in the cache at step \(t\). The monotonicity constraint ensures that once a token is evicted, it cannot re-enter the cache later. The ideal eviction policy solves
\begin{equation}\label{eq:problem}
\min_{r}\ \mathcal{L}(\mathbf{o}'_t;\mathbf{o}_t)
\qquad \text{s.t.} \qquad
\sum_{i=1}^t r_{t,i}\le M.
\end{equation}
That is, we seek a size-\(M\) cache whose attention output remains as close as possible to the full-attention output. Solving this combinatorial problem exactly at every decoding step is infeasible, so most prior work relies on heuristic eviction rules~\citep{xiao2023efficient, han2023lm, zhang2023h2o, li2024snapkv, cai2025r, ghadia2025dialogue}.

\subsection{Token Retention as a Learnable Surrogate}\label{sec:retention}

\citet{bui2025cache} relax the discrete variables \(r_{t,i}\in\{0,1\}\) into continuous retention factors by assuming that each token has an intrinsic importance that decays exponentially over time. Under this relaxation,
\[
\tilde{\alpha}_{t,i}
=
\frac{\beta_i^{\,t-i} e^{z_{t,i}}}
{\sum_{j\in C_t} \beta_j^{\,t-j} e^{z_{t,j}}},
\qquad
\tilde{\mathbf{o}}_t
=
\sum_{i\in C_t}\tilde{\alpha}_{t,i}\mathbf{v}_i,
\]
where \(\beta_i\in[0,1]\) is a learnable retention score for token \(i\). Larger \(\beta_i\) corresponds to greater long-term importance, while \(\beta_i=1\) recovers standard attention. The score \(\beta_i\) is predicted from the token embedding by a small \emph{retention gate}. These gates are trained so that a student model with retention-gated attention matches a full-cache teacher under a memory budget. Let \(p(\cdot\mid x)\) denote the teacher distribution and \(q_\theta(\cdot\mid x)\) the student distribution with gate parameters \(\theta\). The quality loss is
\begin{equation}
\mathcal{L}_{\text{quality}}
=
D_{\mathrm{KL}}(p \,\|\, q_\theta)
+
\mathbb{E}_{(x,y)}[-\log q_\theta(y\mid x)].
\end{equation}
To enforce the capacity constraint, they introduce
\begin{equation}
\mathcal{L}_{\text{cap}}
= \sum_{\ell,h} \sum_{t=1}^{T}
\max\!\left(
0,\,
\sum_{i=1}^{t}\beta_{\ell,h,i}^{\,t-i}
-
M
\right),
\end{equation}
where \(M\) is the KV budget for each attention head. The overall objective is
\begin{equation}
\mathcal{L}
=
\mathcal{L}_{\text{quality}}
+
\lambda \mathcal{L}_{\text{cap}}.
\end{equation}

\section{Can KV Eviction Improve Long-Context Performance?}\label{sec:bey_eff}
In this section, we explain why KV eviction can improve long-context performance through \emph{attention dilution}: full-cache attention spreads mass over many irrelevant tokens, while selective eviction suppresses distractors and concentrates attention on useful context. We further interpret geometric retention as a query-agnostic surrogate for future token utility.

\subsection{Attention Dilution as Loss of Useful Mass}\label{sec:dilution}

\begin{wrapfigure}{r}{0.48\textwidth}
    \vspace{-7mm}
    \centering
    \begin{subfigure}[b]{0.50\linewidth}
        \centering
        \includegraphics[width=\linewidth]{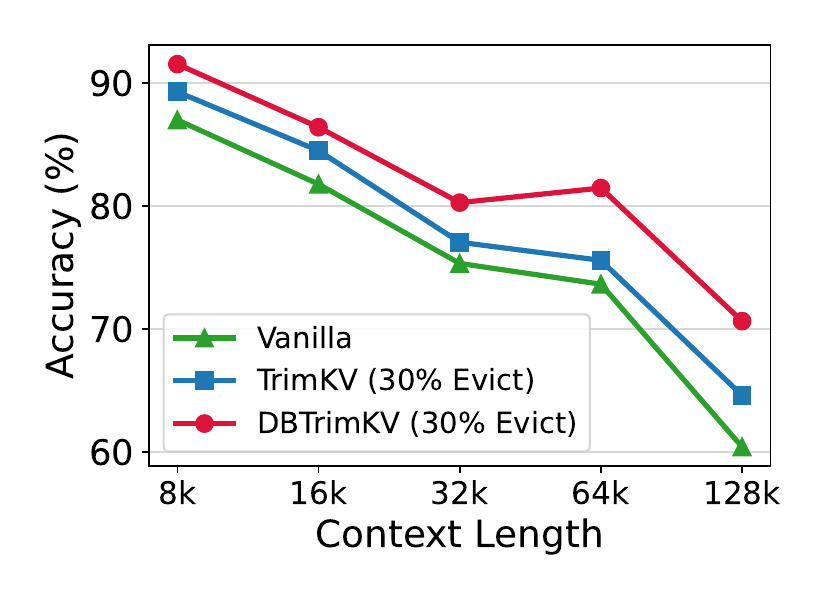}
        \caption{Varying Context Length} 
        \label{fig:niah_context_vs_performance} 
    \end{subfigure}\hfill
    \begin{subfigure}[b]{0.48\linewidth}
        \centering
        \includegraphics[width=\linewidth]{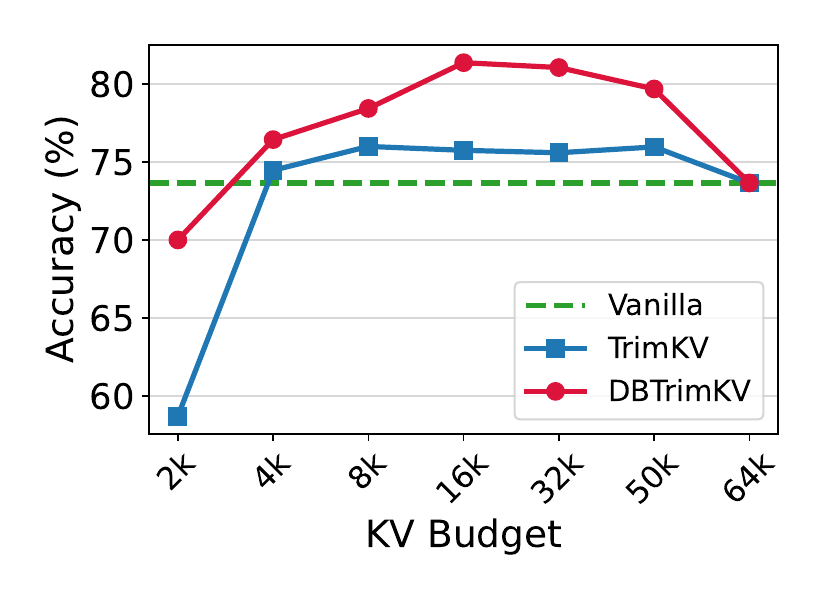}
        \caption{Varying KV Budget}
        \label{fig:niah_kv_budget_vs_performance} 
    \end{subfigure}
    \caption{\textbf{a)}: Multi-key, -value, -needle NIAH accuracy across increasing context lengths. \textbf{b)}: Accuracy at 64k context and varying budgets.} 
    \label{fig:dilution}
    \vspace{-6mm}
\end{wrapfigure}

Long-context failures can arise even when the relevant evidence is present in the cache~\citep{liu2024lost, yang2025llm} (see Figure~\ref{fig:niah_context_vs_performance}). At a decoding step \(t\), only a small subset \(U_t\subseteq C_t\) may be useful for the next prediction, while the remaining tokens act as distractors~\citep{deng2024sparse}. Since self-attention normalizes over the entire cache, useful tokens must compete with all distractors in the softmax denominator. As the number of distractors grows, the total attention mass assigned to useful tokens can vanish.

To make this precise, define the oracle sparse attention that keeps the same logits on
useful tokens but removes all distractors:
\begin{equation}
\alpha^\star_{t,i}
=
\frac{\mathbf{1}\{i\in U_t\}e^{z_{t,i}}}
{\sum_{j\in U_t}e^{z_{t,j}}}.
\end{equation} 
This oracle is motivated by the empirical observation that LLMs often perform well in short-context settings but degrade as the context length increases and irrelevant tokens compete for attention~\citep{bansal2026lets}.

We define the \emph{attention dilution} at step \(t\) as the fraction of attention mass assigned to distractors:
\begin{equation}\label{eq:dilution}
\delta_t
:=
1-\sum_{i\in U_t}\alpha_{t,i}.
\end{equation}
Equivalently, \(\delta_t=\mathrm{TV}(\alpha_t^\star,\alpha_t)\). Thus, the dilution quantity is exactly the distance between full-cache attention and the oracle distribution that attends only to useful tokens.

\paragraph{Near-tie distractors force attention dilution.} The next result shows that severe dilution is unavoidable when many distractors have logits close to those of the useful tokens.

\begin{proposition}[Near-tie distractors force dilution]
\label{prop:tv_dilution}
Fix a decoding step \(t\). Suppose there exist \(\Delta\ge 0\) and a subset \(D'_t\subseteq C_t\setminus U_t\) such that
$z_{t,d}\ge \max_{i\in U_t} z_{t,i}-\Delta$ for all $d\in D'_t$.
Then
\[
\delta_t \ge
\frac{e^{-\Delta}|D'_t|/|U_t|}
{1+e^{-\Delta}|D'_t|/|U_t|}.
\]
Consequently, if \(\Delta=O(1)\), \(|U_t|=O(1)\), and \(|D'_t|\to\infty\), then $\delta_t \to 1$ and  $\sum_{i\in U_t}\alpha_{t,i}\to 0$.
\end{proposition}

The proof is given in Appendix~\ref{sec:dilut_proof}. This proposition shows that the total attention mass assigned to useful tokens vanishes as the number of competitive distractors grows under the stated assumptions. Although individual distractors may receive little attention, their cumulative contribution to the softmax denominator can dominate.

\paragraph{Eviction mitigates dilution.}
We show that selective eviction can mitigate the dilution effect defined in Eq.~\eqref{eq:dilution}.  Let \(r_{t,i}\in[0,1]\) be a general retention weight and $\alpha^{r}_{t, i}$ be the retention-gated attention in Eq~\eqref{eq:reform}. Here, hard eviction is the special case \(r_{t,i}\in\{0,1\}\) and geometric retention is \(r_{t,i}=\beta_i^{t-i}\). 

\begin{corollary}[Preferential retention reduces dilution]
\label{prop:retention_reduces_dilution}
For any decoding step \(t\), we have
\[
\delta^r_t
=
\frac{(\rho_D/\rho_U)\delta_t}
{(1-\delta_t)+ (\rho_D/\rho_U)\delta_t}, \qquad \text{where} \quad \rho_U
=
\frac{\sum_{i\in U_t}r_{t,i}e^{z_{t,i}}}
{\sum_{i\in U_t}e^{z_{t,i}}},
\quad
\rho_D
=
\frac{\sum_{i\notin U_t}r_{t,i}e^{z_{t,i}}}
{\sum_{i\notin U_t}e^{z_{t,i}}}.
\]
Consequently, if \(\rho_D\le \rho_U\), then $\delta^r_t \le \delta_t$. If the ratio \(\rho_D/\rho_U \to 0\), then $\delta^r \to 0$.
\end{corollary}

Corollary~\ref{prop:retention_reduces_dilution} formalizes when preferential retention reduces attention dilution. Its condition, \(\rho_U \ge \rho_D\), only requires useful tokens to be retained at a higher logit-weighted average rate than distractors. Thus, any retention rule that suppresses distractor mass more than useful-token mass reduces distractor-induced dilution. Here, the useful-token set is an oracle construct for analysis. Reduced dilution does not by itself guarantee improved task performance. From this perspective, KV eviction is not merely an efficient approximation to full-cache attention: when irrelevant context dilutes attention, selective eviction can serve as a corrective mechanism. The empirical results for learnable eviction methods in Figure~\ref{fig:dilution} are consistent with this prediction, as the best-performing model is not the full-cache model.

\subsection{Geometric Retention as Query-Agnostic Future Utility}
\label{sec:geom_retention}

We now justify the geometric retention form \(r_{t,i}=\beta_i^{\,t-i}\) used by retention-gated attention. The key intuition is that token importance is both sparse and local.  Many tokens are important for the current query or nearby queries, but their utility fades quickly once the generation moves to a different subproblem, entity, or topic.  Other tokens remain useful for much longer.  Thus, the quantity we want for KV eviction is not only the current attention score of token \(i\), but its future persistence: how likely it is to keep receiving non-negligible attention as decoding continues.

At decoding step \(t\), we view compression as a one-shot decision over the old cache \(C_t=\{1,\dots,t\}\). Future tokens are not part of this decision and can be handled by later compression rounds. For an old token \(i\le t\), define its cumulative future utility as
\begin{equation}
\label{eq:future_utility}
\bar G_i(t)
:=
\sum_{s=t+1}^T
w_{t,s}
\Pr(i\in U_s^{(t)} \mid \mathcal F_t),
\end{equation}
where \(w_{t,s}\ge 0\) are horizon weights and \(U_s^{(t)}\subseteq C_t\) denotes the old cached tokens that remain useful at future step \(s\). Here, $\mc F_t$ summarizes information available at step $t$.

To reason about this probability, consider a fixed attention head and approximate future query--key compatibility in a low-dimensional query state:
\[
z_{s,i}
=
\mathbf x_s^\top(\mathbf W_Q^\top \mathbf W_K)\mathbf x_i
\approx
\mathbf r_s^\top \mathbf c_i,
\qquad
\mathbf r_s,\mathbf c_i\in\mathbb R^m,\quad m\ll d .
\]
Here \(\mathbf r_s\) is the future query state and \(\mathbf c_i\) is a fixed compatibility vector for token \(i\). For a query state \(\mathbf r\), let
\[
\Gamma_{K,t}(\mathbf r)
:=
K\text{-th largest value in }
\{\mathbf r^\top \mathbf c_j:j\le t\}
\]
be the old-cache top-\(K\) threshold. Token \(i\) is immediately useful at step $s$ if \(\mathbf r^\top \mathbf c_i\ge \Gamma_{K,t}(\mathbf r)\).

Immediate top-\(K\) membership, however, is too strict for retention. A token may temporarily fall below the top-\(K\) boundary but may become useful again when generation returns to the same entity, instruction, document, or topic. We therefore define a relaxed survival region
\[
\mathcal S_{i,t}^{(K)}
=
\left\{
\mathbf r:
\mathbf r^\top \mathbf c_i
\ge
\Gamma_{K,t}(\mathbf r)-\Delta_{i,t}(\mathbf r)
\right\},
\qquad
\Delta_{i,t}(\mathbf r)\ge 0 .
\]
The slack \(\Delta_{i,t}(\mathbf r)\) allows token \(i\) to remain retention-worthy even when it is not currently in the hard top-\(K\). It is token dependent: local tokens may admit only a small slack, whereas globally useful tokens, such as topic-summary tokens~\citep{mu2023learning}, delimiter tokens, or other persistent structural tokens, may tolerate a larger drop below the top-\(K\) boundary. For a fixed compression time \(t\) and a future step \(s>t\), define
\[
U_s^{(t)}
:=
\left\{
i\le t:
\mathbf r_u\in\mathcal S_{i,t}^{(K)}
\ \text{for all }u=t+1,\dots,s
\right\}.
\]

This definition separates immediate usefulness from future retention-worthiness. The hard top-\(K\) region captures tokens that are among the strongest competitors for the current query state, while the relaxed region captures tokens that remain relevant to become useful again. Since KV eviction is monotone, a removed token cannot re-enter the cache; retention should therefore estimate whether future query states are likely to remain inside this survival region.

The following result justifies the geometric decay of the retention under these definitions.

\begin{theorem}[Geometric decay of retention (Informal)]
\label{thm:geom_relaxed_topk_informal}
Assume that, within an attention head, future query states $(\ve r_s)_{s > t}$ evolve according to stable dynamics. If, whenever the query state is inside token \(i\)'s relaxed top-\(K\) region \(\mathcal S_{i,t}^{(K)}\), there is probability at least \(\epsilon_i>0\) of exiting this region within the next \(b_i\) decoding steps, then there exists $A_i(t)$ such that
\[
\bar G_i(t)
\le
A_i(t)
\sum_{s=t+1}^T
w_{t,s}
\beta_i^{\,s-t}, \qquad
\beta_i=(1-\epsilon_i)^{1/b_i}\in(0,1).
\]
\end{theorem}

The parameter \(\beta_i\) therefore summarizes the persistence of token \(i\): short-lived local tokens have small \(\beta_i\), while globally useful or structural tokens have \(\beta_i\) close to one. We give the formal statement and proof in Appendix~\ref{app:geom_retention}.

Empirically, token persistence exhibits this fast-decay structure under full-cache inference. We run Qwen3-VL-4B on 98 long MMDU multimodal multi-turn dialogues with interleaved text and images, prefilled to \(T=16{,}384\) tokens. For each decoder layer and head, we record which past tokens are selected by each query. A token is counted as alive at horizon \(L\) if it is selected by some query at least \(L\) positions after its birth. Figure~\ref{fig:survival} shows that survival drops rapidly with \(L\): for the shown head, only \(13.8\%\) of tokens survive to \(L=4096\) under top-\(K=256\), and only \(3.2\%\) survive to \(L=16\mathrm{k}\). Even under the more lenient \(99\%\)-mass criterion, the head loses \(60\%\) of mass-relevant tokens by \(L=4096\). Vision tokens, which dominate the multimodal cache, fade at a similar rate as text tokens. In this regime, retaining every old token can dilute attention mass, while evicting tokens whose future utility has already decayed can recover capacity for tokens that remain relevant.

\begin{wrapfigure}{r}{0.53\textwidth}
    \vspace{-5mm}
    \centering
    \begin{subfigure}[b]{0.49\linewidth}
        \centering
        \includegraphics[width=\linewidth]{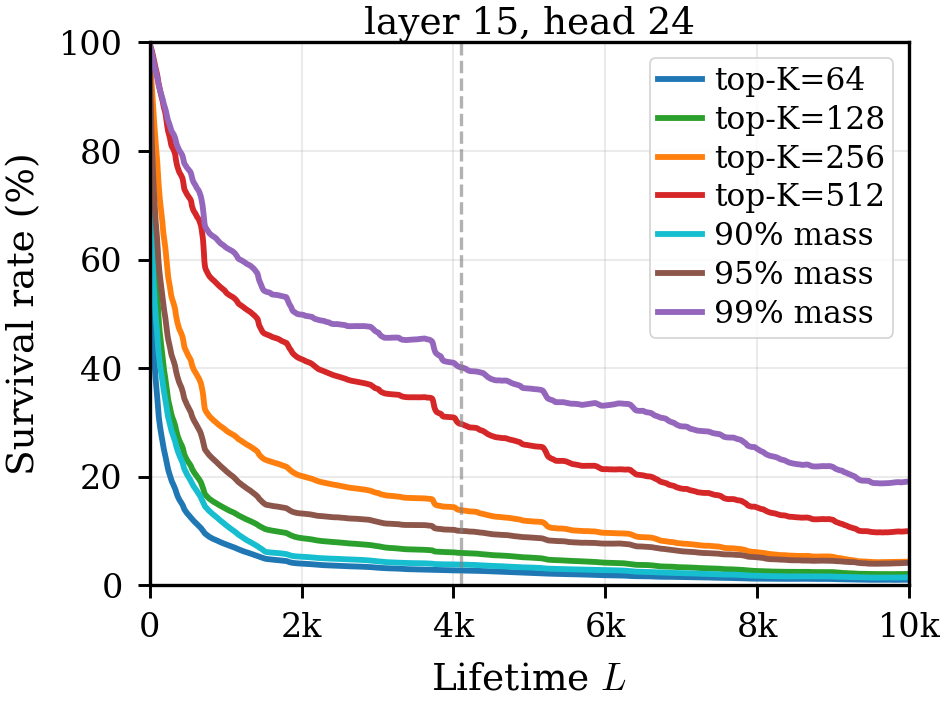}
        \caption{Survival curves}
        \label{fig:survival_curves}
        \vspace{-2mm}
    \end{subfigure}
    \hfill
    \begin{subfigure}[b]{0.49\linewidth}
        \centering
        \includegraphics[width=\linewidth]{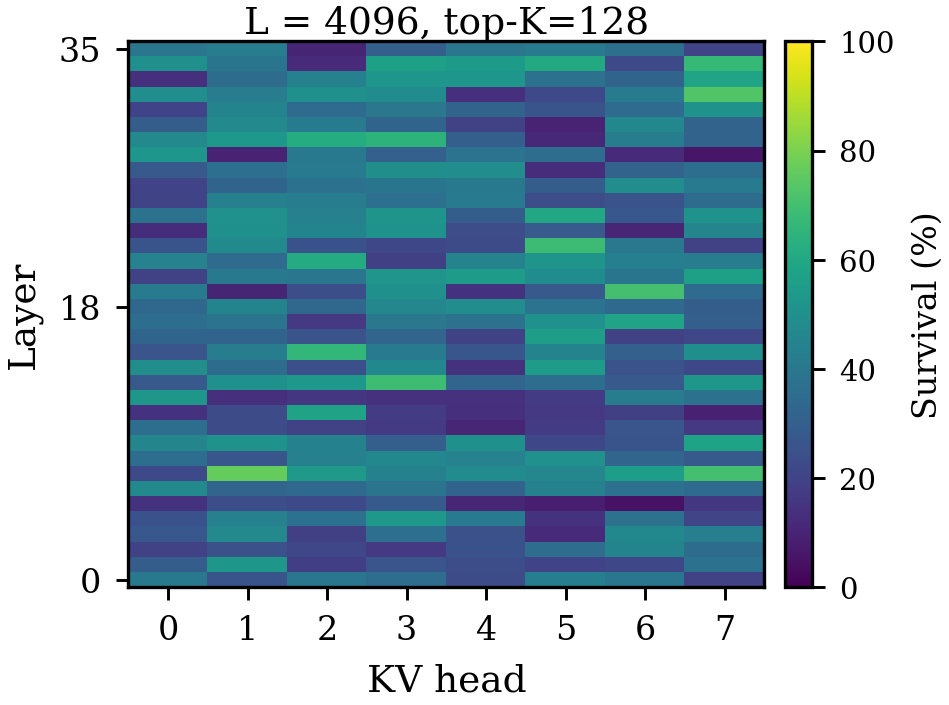}
        \caption{Per-head survival}
        \label{fig:survival_heatmap}
        \vspace{-2mm}
    \end{subfigure}
    \vspace{-4mm}
    \caption{\textbf{a)} Fraction of tokens whose attention persists to horizon
    \(L\) for layer 15, head 24. \textbf{b)} Per-layer, per-head survival at
    \(L=4096\) and top-\(K=128\).}
    \label{fig:survival}
    \vspace{-3mm}
\end{wrapfigure}
These results support geometric retention as a simple surrogate for future token persistence. In practice, we do not explicitly estimate the exit parameters \((b_i,\epsilon_i)\). Instead, the retention gate predicts \(\beta_i\) directly from the token representation, allowing the model to learn which tokens should decay quickly and which should persist. To avoid recomputing \(\beta_i\) at every compression step, we estimate it only once when the token enters the cache. The resulting retention weight then follows the geometric form
\[
r_{t,i} = \beta_i^{\,t-i},
\]
which can be interpreted as the predicted probability that token \(i\) remains useful at step \(t\) given the information available when it was first cached.

\begin{figure*}
    \centering
   \includegraphics[width=0.9\linewidth]{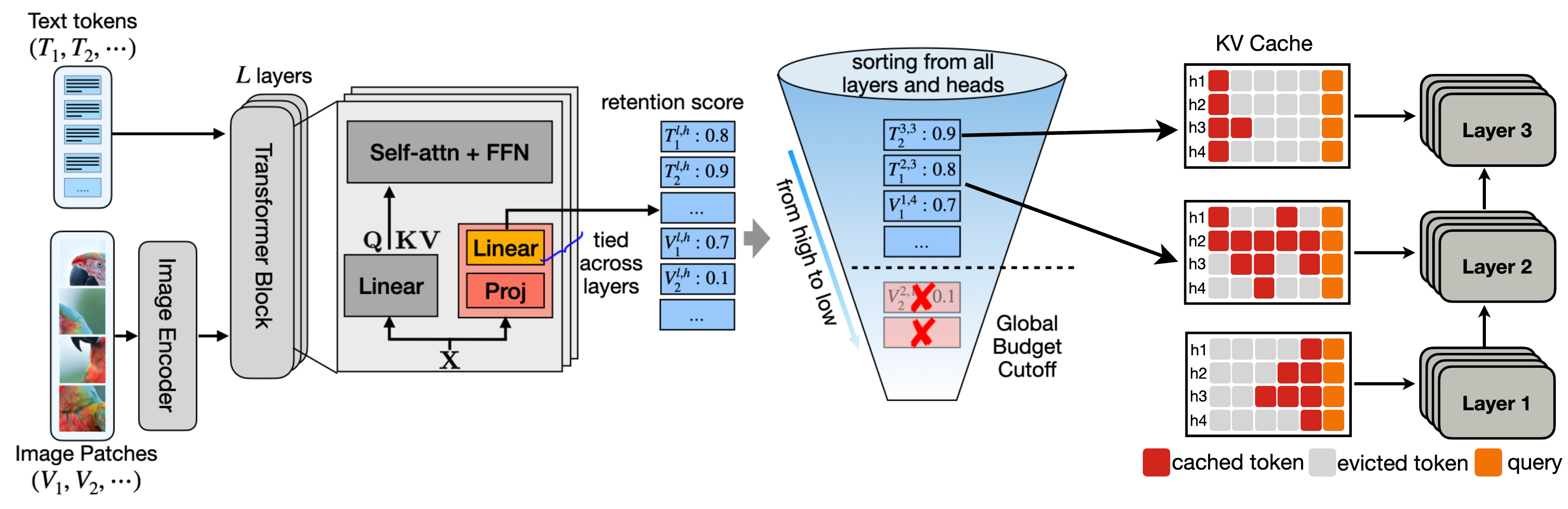}
   \vspace{-4mm}
    \caption{Prefill/compression with our method. A per-layer, per-head retention gate assigns each text/visual token a scalar score; the gate’s final projection is tied across layers and heads to encourage score comparability. Tokens from all layers, heads, and modalities are then pooled and globally ranked by retention, replacing fixed per-layer/per-head KV budgets.}
    \label{fig:main}
   \vspace{-3mm}
\end{figure*}

\section{Global Token Retention via Weight-Tied Gates}\label{sec:method}

The per-head survival heatmap in Figure~\ref{fig:survival_heatmap} shows strong heterogeneity: a small number of heads preserve long-range tokens, while many
heads quickly lose them.  This raises a key question: \textit{"how should a limited KV budget be allocated across layers and heads?"}. Existing methods typically use fixed per-layer or per-head budgets~\citep{bui2025cache}, or adaptively allocate budgets using a myopic attention heuristics~\citep{feng2024ada}. A natural alternative is to rank tokens by retention, but standard retention gates are trained independently across layers and heads, so their scores are not directly comparable. 

To address this, we introduce \emph{global} KV eviction via weight-tied retention gates. Specifically, we use per-layer, per-head gates whose \emph{final scoring projection is shared} across all layers and heads, placing retention scores on a common scale. This allows all KV entries to be ranked globally under a single cache budget, replacing hand-designed per-layer or per-head allocations with a unified eviction rule. 

\subsection{Architecture and Training}
\paragraph{Weight-tied retention gates.}
Consider a transformer with $L$ layers and $H$ attention heads per layer. For each token at position $t$, layer $\ell$, and head $h$, we predict a retention coefficient
\[
\beta_{\ell,h,t} = g_{\ell,h}(\mathbf{x}_t)\in[0,1],
\]
where $\mathbf{x}_t$ is the token embedding. Our main design choice is the parameterization of $g_{\ell,h}$. Each gate first computes a head-specific embedding, but the final scalar score is produced by a \emph{shared} projection:
\[
g_{\ell,h}(\mathbf{x})
=
\sigma\!\left(\ve w_g^\top \mathrm{Proj}_{\ell,h}(\mathbf{x}) + b_g\right),
\]
where $\mathrm{Proj}_{\ell,h}$ is layer/head-specific, while $(\ve w_g,b_g)$ is tied across all layers and heads. This shared readout encourages retention scores to be comparable across layers and heads. Without such tying, differences in score scales may make global eviction less reliable.

% During training, these coefficients are used exactly as in retention-gated attention, through the exponential decay factor $\beta_{\ell,h,i}^{\,t-i}$. Thus, the model learns which tokens are expected to retain utility over long horizons, while the shared final projection ensures that these utilities are measured on a common scale.

\paragraph{Training.} We follow the same training procedure as~\citep{bui2025cache} described in Section~\ref{sec:retention} but replace the local capacity constraint by a global counterpart
\[
\mathcal{L}_{\text{cap}}
=
\sum_{t=1}^{T}
\max\!\left(
0,\,
\sum_{\ell,h}\sum_{i=1}^{t}\beta_{\ell,h,i}^{\,t-i}
-
M_{\mathrm{global}}
\right),
\]
where $M_{\mathrm{global}}$ is the target global KV budget. We only train retention gates and freeze LLM weights.

\subsection{Global KV Eviction}\label{sec:infer}
At inference time, we use the learned retention coefficients to rank all cached KV entries globally. For each cached token $(\ell,h,i)$ present at decoding step $t$, we assign the retention score
\begin{equation}
\label{eq:global_score}
\widetilde G_{\ell,h,i}(t)
=
\sum_{s=t+1}^{T}\beta_{\ell,h,i}^{\,s-i}
=
\beta_{\ell,h,i}^{\,t+1-i}\,
\frac{1-\beta_{\ell,h,i}^{\,T-t}}{1-\beta_{\ell,h,i}}, \quad \text{for} \quad
\beta_{\ell,h,i} < 1.
\end{equation}
Here, we choose the simple horizon weight $w_{t, s} = 1$. Thus, $\widetilde G_{\ell,h,i}(t)$ aggregates the predicted future utility of token $(\ell,h,i)$ over the remaining decoding horizon $T-t$. In the one-step case $T-t=1$, it reduces to the myopic score $\widetilde G_{\ell,h,i}(t)=\beta_{\ell,h,i}^{\,t+1-i}$, used in~\citep{bui2025cache}.
More generally, for $\beta_{\ell,h,i}\in[0,1)$ and $T-t\to\infty$, then $\textstyle\widetilde G_{\ell,h,i}(t) \rightarrow
\frac{\beta_{\ell,h,i}^{\,t+1-i}}{1-\beta_{\ell,h,i}}$.

The lookahead horizon \(T-t\) therefore governs the trade-off between recency and retention. Recency enters through the factor \(\beta_i^{\,t+1-i}\), while longer horizons increasingly favor tokens with larger retention parameters \(\beta_i\). This softens the effective local sliding-window and geometric decay bias and allows older, slowly decaying tokens to remain competitive.

Our eviction rule is simple now: retain the \(M_{\mathrm{global}}\) tokens with the largest scores \(\tilde G_{\ell,h,i}(t)\) across all layers and heads. Unlike existing budget-allocation methods, this requires no predefined per-layer or per-head budgets. Instead, capacity is assigned automatically by a unified retention ranking.

\begin{wrapfigure}{r}{0.48\textwidth}
    % \vspace{-2mm}
    \centering
    \includegraphics[width=\linewidth]{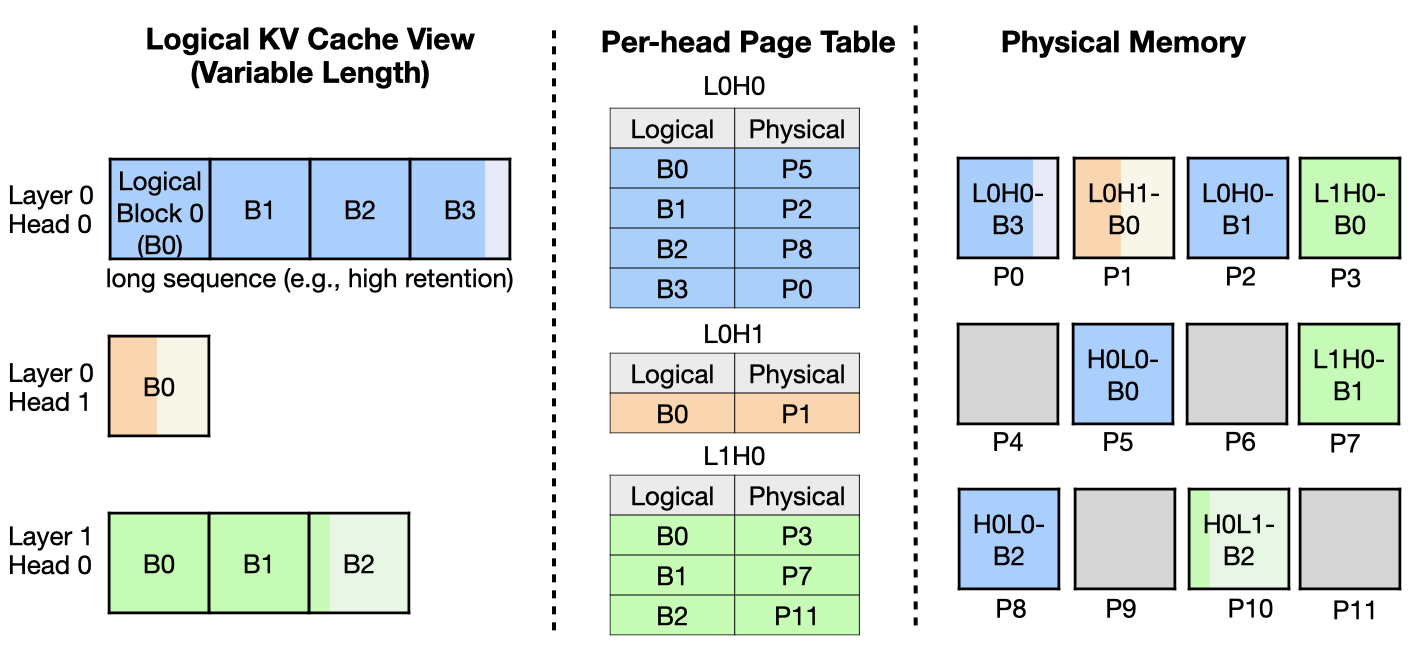}
    \caption{PagedAttention with head-specific variable-length KV caches. Each layer–head maintains a logical KV sequence of variable length, stored as fixed-size blocks.}
    \vspace{-4mm}
    \label{fig:paged_attention}
\end{wrapfigure}
\paragraph{Implementation details.} To support variable cache sizes induced by dynamic, head-specific budgets, we use a paged-attention layout, similar to~\citet{kwon2023efficient}. Specifically, KV entries are stored in fixed-size pages, and each head maintains a block table for its currently active pages. Attention is then computed with variable-length kernels~\citep{dao2023flashattention} using the resulting per-head sequence lengths. This allows each head to maintain a variable-length logical KV sequence without requiring contiguous physical storage or kernel specialization. We use physical blocks of 256 tokens and allocate additional blocks dynamically when the required capacity exceeds the current allocation, without requiring a fixed maximum block count in advance. An example is provided in Figure~\ref{fig:paged_attention}.

\section{Experiments}

We evaluate our approach across benchmarks covering: (i) long-horizon reasoning, (ii) multi-turn dialogue, (iii) short-form question answering, and (iv) long-context question answering. Throughout, we refer to our method as \emph{Dynamic Budget TrimKV} (DBTrimKV).

\subsection{Short-Form Question Answering}
\label{appendix:section:subsection:short-form_question_answering}
\begin{table}[t]
  \centering
  \resizebox{\linewidth}{!}{
  \begin{tabular}{l cccc cccc cccc}
    \toprule
    \multirow{2}{*}{\textbf{Method}} & \multicolumn{4}{c}{\textbf{128 Visual Tokens}} & \multicolumn{4}{c}{\textbf{64 Visual Tokens}} & \multicolumn{4}{c}{\textbf{32 Visual Tokens}} \\
    \cmidrule(lr){2-5} \cmidrule(lr){6-9} \cmidrule(lr){10-13}
    & \textbf{GQA} & \textbf{VQAText} & \textbf{MME} & \textbf{Avg. (\%)} & \textbf{GQA} & \textbf{VQAText} & \textbf{MME} & \textbf{Avg. (\%)} & \textbf{GQA} & \textbf{VQAText} & \textbf{MME} & \textbf{Avg. (\%)} \\
    \midrule
    \rowcolor{gray!20} \cellcolor{white}
    Vanilla* & 61.9 & 58.2 & 1506.5 & 100.0 & 61.9 & 58.2 & 1506.5 & 100.0 & 61.9 & 58.2 & 1506.5 & 100.0 \\
    \midrule
    FastV* & - & - & - & - & 46.0 & 51.6 & 973.5 & 75.9 & - & - & - & - \\
    DART* & 57.9 & 56.3 & 1408.7 & 94.6 & 54.7 & 54.7 & 1365.1 & 91.0 & 52.9 & 52.2 & 1273.3 & 86.6 \\
    PruMerge* & 58.2 & 54.0 & 1408.1 & 93.4 & 55.4 & 52.0 & 1316.8 & 88.8 & 52.9 & 49.2 & 1236.6 & 84.0 \\
    VisionZip* & 57.6 & 56.9 & 1436.9 & 95.4 & 55.1 & 55.5 & 1365.2 & 91.7 & 51.8 & 53.1 & 1251.2 & 86.0 \\
    CDPruner* & 59.9 & 56.2 & 1431.4 & 96.1 & 58.6 & 55.3 & 1415.1 & 94.5 & 57.0 & 53.2 & 1373.0 & 91.5 \\
    DivPrune* & 59.4 & 55.9 & 1405.1 & 95.1 & 57.5 & 54.5 & 1334.7 & 91.7 & 54.9 & 52.9 & 1284.9 & 88.3 \\
    \trimkv & 61.2 & 57.7 & 1490.5 & 99.9 & 61.2 & 57.3 & 1490.5 & 99.7 & 61.1 & 56.8 & 1490.5 & 99.4 \\
    \rowcolor{\ourclr!20}  \cellcolor{white}\dbtrimkv & 61.2 & 57.7 & 1490.5 & 99.9 & 61.1 & 57.3 & 1490.5 & 99.7 & 61.0 & 57.0 & 1490.5 & 99.4 \\
    \bottomrule
  \end{tabular}
  }
  \caption{Performance of competing method on LLaVA-1.5-7B and short-form question answering tasks. The symbol * indicates that the numbers are copied from~\citep{zhang2025beyond}. Our result is reproduced by running with the LMMs-Eval framework. In the Avg. column, we report the relative performance compared to the corresponding vanilla inference.}
  \label{tab:shortgen-llava15}
\end{table}

We first study short-form visual question answering, where a relatively small number of output tokens must be generated conditioned on an image and a short textual prompt.  In this experiment, we compare DBTrimKV and TrimKV against visual token pruning methods that operate exclusively at the prefilling stage and remove only a subset of visual tokens before decoding. To ensure a fair comparison, we fix the total KV budget of DBTrimKV and TrimKV to be equal to the sum of (i) all text tokens in the prompt and (ii) the visual-token budget used by each pruning baseline.

\textbf{Benchmarks.}
We evaluate on three standard short-form VQA benchmarks: VQAText~\citep{singh2019towards}, MME~\citep{fu2025mme}, and GQA~\citep{hudson2019gqa}. These benchmarks cover a range of visual reasoning skills, from simple recognition to compositional and relational reasoning, while keeping generation lengths short.

\textbf{Baselines.} We compare our method against recent visual token pruning approaches, including FastV~\citep{chen2024image}, VisionZip~\citep{yang2025visionzip}, DART~\citep{wen2025stop}, PruMerge~\citep{shang2025llava}, DivPrune~\citep{alvar2025divprune}, and CDPruner~\citep{zhang2025beyond}. All these methods reduce memory usage by selectively pruning visual tokens during prefilling, while leaving the textual KV cache intact throughout decoding. Additional comparisons with CAKE and HeadKV on MathVision\textsubscript{mini} are reported in Appendix~\ref{appendix:additional_baselines}.

\textbf{Implementation details.} 
For this experiment, we use LLaVa-1.5-7B as the base model. We train our method on the LLaVA-Next dataset~\citep{liu2024llavanext}. Following~\citep{bui2025cache}, we update only the retention-gate weights while keeping all original model parameters frozen. We set the objective hyperparameter to $\lambda_{\mathrm{cap}} = 1.0$ and the memory capacity to $M_\mathrm{global} = 64 * L * H$, where $L$ and $H$ are the number of layers and heads, respectively. Each transformer block is equipped with a retention gate $g$, implemented as a small MLP with hidden dimension $d_g = 512$. For the retention gate architecture, we implement the $\mathrm{Proj}$ as a two-layer MLP. During the training, we tie the final linear layer $(\ve W, b)$ across all layers and heads as described in Section~\ref{sec:method}. The bias term $b$ is initialized to 18.0 at the start of training. During inference, we use a default lookahead horizon $T-t =2$ for every eviction time $t$. Appendix~\ref{appendix:ood_training} evaluates out-of-distribution transfer by training retention gates on general QA data and testing on AIME24.

\textbf{Results.}
As shown in Table~\ref{tab:shortgen-llava15}, both DBTrimKV and TrimKV consistently outperform visual token pruning baselines and match vanilla performance under all memory budgets. This gap highlights a key limitation of prefilling-only visual pruning: by constraining memory exclusively on the vision side, these methods cannot adapt to how different heads and layers actually utilize visual versus textual information during decoding. In contrast, DBTrimKV and TrimKV compress the KV cache jointly over both vision and text tokens and across all layers and heads. This flexibility allows the model to retain more visual tokens in heads that encode rich visual semantics, while aggressively trimming tokens, visual or textual, that contribute less to downstream prediction.

\begin{table*}[t]
  \centering
  \resizebox{0.98\textwidth}{!}{
    \begin{tabular}{llcccccccc}
      \toprule
      & & \multicolumn{3}{c}{\textbf{Image Reasoning}} & \multicolumn{4}{c}{\textbf{Video Reasoning}} & \textbf{Average} \\
      \cmidrule(lr){3-5}\cmidrule(lr){6-9}
      \textbf{Budget} & \textbf{Method} 
      & MMStar & MathVision\textsubscript{mini} & MMMUPro\textsubscript{vision} 
      & VideoMME & VideoMMMU\textsubscript{adap} & VideoMMMU\textsubscript{comp} & VideoMathQA\textsubscript{mcq} 
      & (\% vs Vanilla) \\
      \midrule

      % ===================== VANILLA =====================

  \rowcolor{gray!20} \cellcolor{white}

  & Vanilla 

  & 71.52 & 48.68 & 40.64 

  & 54.22 & 35.67 & 55.00 & 36.19 

  & 100 \\

  \midrule

  % ===================== 1024 =====================

  \multirow{6}{*}{1024}

    & SnapKV          & 51.84 & 15.13 & 18.27 & 49.22 & 21.00 & 28.67 & 20.24 & 58.03 \\

    & R-KV            & 58.42 & 24.67 & 26.24 & 51.48 & 22.00 & 30.33 & 26.43 & 68.82 \\

    & AdaKV           & 66.89 & 32.89 & 25.61 & 52.00 & 25.42 & \underline{31.67} & 23.57 & 73.43 \\

    & Ada-Pyramid-KV  & 66.84 & 28.29 & 25.90 & 53.81 & 28.96 & 30.74 & 23.09 & 73.63 \\

     \cellcolor{white}

    & \trimkv          & \underline{70.64} & \underline{45.72} & \underline{34.34} & \underline{54.15} & \underline{35.00} & \underline{\textcolor{TinaCrimson}{59.00}} & \underline{35.00} & \underline{97.02} \\

  \rowcolor{\ourclr!20} \cellcolor{white}

    & \dbtrimkv        & \textbf{71.50} & \textbf{\textcolor{TinaCrimson}{52.63}} & \textbf{\textcolor{TinaCrimson}{41.1}} & \textbf{\textcolor{TinaCrimson}{54.93}} & \textbf{\textcolor{TinaCrimson}{37.00}} & \textbf{\textcolor{TinaCrimson}{59.33}} & \textbf{\textcolor{TinaCrimson}{36.43}} & \textbf{\textcolor{TinaCrimson}{103.26}} \\

  \midrule
  % ===================== 512 =====================

  \multirow{6}{*}{512}

    & SnapKV          & 51.01 & 14.80 & 9.36 & 49.22 & 22.74 & 21.40 & 18.09 & 52.60 \\

    & R-KV            & 55.23 & 20.72 & 15.54 & 49.51 & 20.40 & 21.84 & 19.76 & 57.26 \\

    & AdaKV           & 60.14 & 17.11 & 16.99 & 53.77 & 22.15 & 24.67 & 19.05 & 59.97 \\

    & Ada-Pyramid-KV  & 60.41 & 15.46 & 15.61 & 52.63 & 19.73 & 19.46 & 17.14 & 55.68 \\

    \cellcolor{white}

    & \trimkv          & \underline{70.28} & \underline{42.11} & \underline{34.34} & \textbf{\textcolor{TinaCrimson}{54.74}} & \underline{34.00} & \underline{51.51} & \underline{32.86} & \underline{92.85} \\

  \rowcolor{\ourclr!20} \cellcolor{white}

    & \dbtrimkv        & \textbf{\textcolor{TinaCrimson}{71.85}} & \textbf{\textcolor{TinaCrimson}{51.97}} & \textbf{\textcolor{TinaCrimson}{43.87}} & \underline{\textcolor{TinaCrimson}{54.41}} & \textbf{35.57} & \textbf{\textcolor{TinaCrimson}{61.33}} & \textbf{35.95} & \textbf{\textcolor{TinaCrimson}{103.73}} \\

  \midrule
  % ===================== 256 =====================

  \multirow{6}{*}{256}

    & SnapKV          & 50.89 & 7.89 & 6.24 & 50.15 & 17.11 & 14.92 & 14.52 & 44.35 \\

    & R-KV            & 53.01 & 16.12 & 12.37 & 51.26 & 14.38 & 17.27 & 18.10 & 50.56 \\

    & AdaKV           & 53.30 & 8.55 & 8.27 & 52.15 & 17.41 & 18.77 & 14.76 & 47.48 \\

    & Ada-Pyramid-KV  & 52.56 & 8.22 & 8.09 & 50.52 & 16.90 & 15.17 & 15.95 & 46.07 \\

    \cellcolor{white}

    & \trimkv          & \underline{65.67} & \underline{40.79} & \underline{31.45} & \underline{53.15} & \underline{30.67} & \underline{42.48} & \underline{28.81} & \underline{84.83} \\

  \rowcolor{\ourclr!20} \cellcolor{white}

    & \dbtrimkv        & \textbf{69.91} & \textbf{\textcolor{TinaCrimson}{51.64}} & \textbf{\textcolor{TinaCrimson}{41.85}} & \textbf{54.15} & \textbf{\textcolor{TinaCrimson}{39.00}} & \textbf{\textcolor{TinaCrimson}{57.00}} & \textbf{32.86} & \textbf{\textcolor{TinaCrimson}{101.49}} \\

  \midrule

  % ===================== 128 =====================

  \multirow{6}{*}{128}

    & SnapKV          & 40.43 & 1.64 & 2.66 & 46.15 & 14.58 & 13.06 & 11.90 & 35.58 \\

    & R-KV            & 50.04 & 8.22 & 4.51 & 46.56 & 11.38 & 11.03 & 14.29 & 39.32 \\

    & AdaKV           & 40.89 & 2.96 & 4.14 & 47.44 & 16.38 & 9.59 & 13.33 & 37.30 \\

    & Ada-Pyramid-KV  & 41.60 & 3.62 & 4.05 & 45.78 & 15.99 & 13.57 & 12.86 & 37.86 \\

    & \trimkv  & \underline{62.29} & \underline{27.30} & \underline{24.34} & \underline{52.55} & \underline{22.90} & \underline{28.38} & \underline{22.14} & \underline{68.13} \\

  \rowcolor{\ourclr!20} \cellcolor{white}

    & \dbtrimkv        & \textbf{68.69} & \textbf{47.04} & \textbf{40.41} & \textbf{\textcolor{TinaCrimson}{54.85}} & \textbf{34.33} & \textbf{\textcolor{TinaCrimson}{57.00}} & \textbf{34.29} & \textbf{98.27} \\

    \midrule
    % ===================== 64 =====================
    \cellcolor{white}
        & \trimkv      & \underline{54.12} & \underline{12.17} & \underline{13.24} & \underline{49.74} & \underline{18.40} & \underline{18.18} & \underline{19.52} & \underline{51.94} \\
      \rowcolor{\ourclr!20}
      \multirow{-2}{*}{\cellcolor{white}64}
        & \dbtrimkv      & \textbf{65.44} & \textbf{33.88} & \textbf{32.49} & \textbf{54.15} & \textbf{28.67} & \textbf{42.67} & \textbf{25.71} & \textbf{81.42} \\
        
      \bottomrule
    \end{tabular}
  }
  \caption{Results across image and video reasoning benchmarks. The best eviction methods are shown in \textbf{bold}, and the second-best are \underline{underlined}. Performances that exceed vanilla inference are highlighted in \textcolor{TinaCrimson}{\textbf{red}}. The superscript \textsuperscript{\orgfire} denotes a trainable KV eviction method.}
  \label{tab:main}
\end{table*}

\subsection{Long-Horizon Reasoning}\label{sec:exp_long}

% Short-form QA mainly stresses memory during prefilling, whereas
Real VLM applications require long-context, long-horizon generation where KV pressure accumulates over many decoding steps. In this regime, eviction must preserve long-lived evidence, and suboptimal allocation compounds over time. We therefore evaluate on long-horizon reasoning benchmarks that amplify these effects through both longer contexts and longer generations.

\begin{figure*}[t]
    \centering
    \begin{subfigure}{0.245\linewidth}
        \centering
        \includegraphics[width=0.9\linewidth]{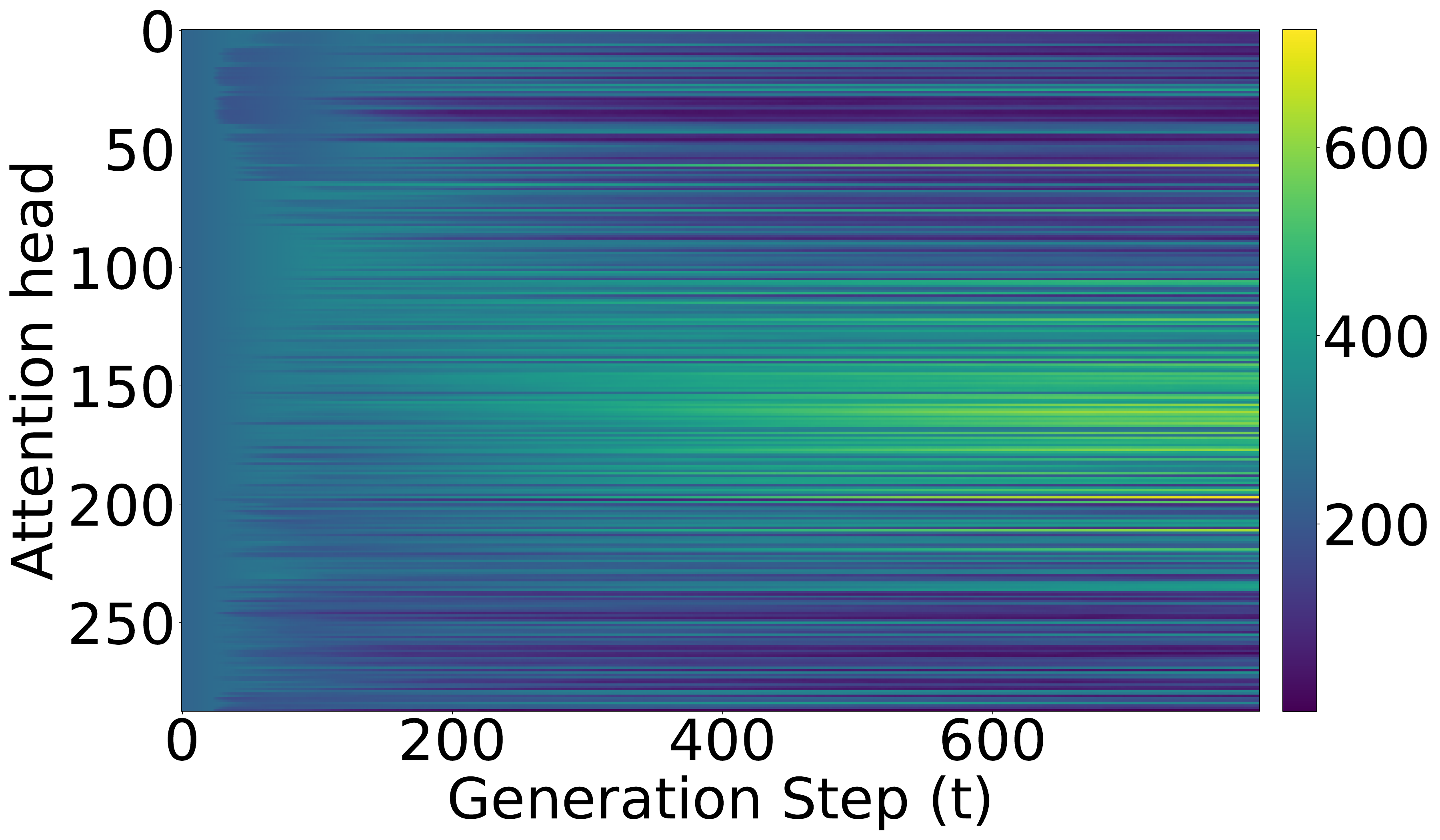}
        \caption{\#tokens}
        \label{fig:kv_cache_usage:total_tokens}
    \end{subfigure}
    \begin{subfigure}{0.245\linewidth}
        \centering
        \includegraphics[width=0.9\linewidth]{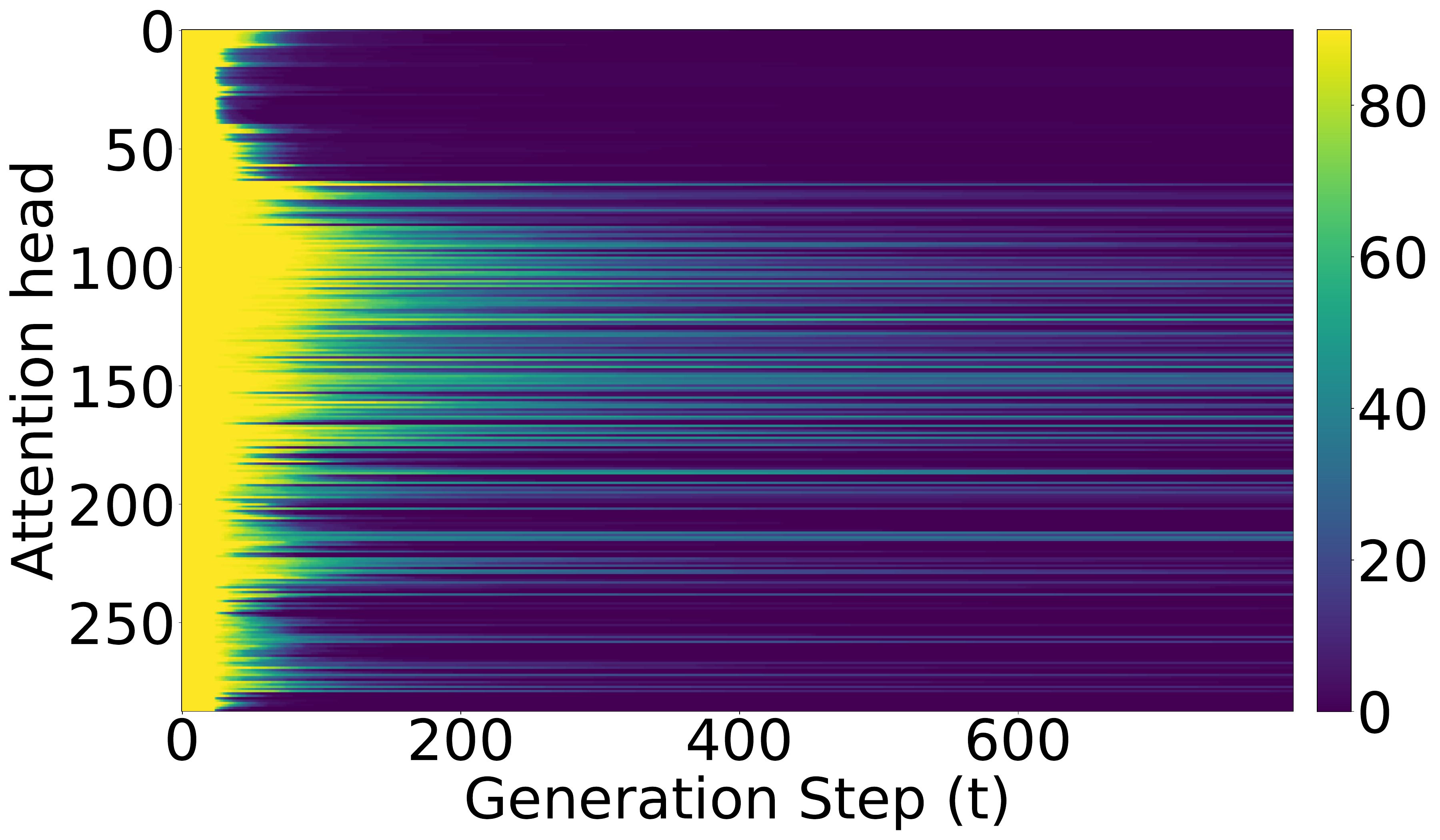}
        \caption{\#visual tokens}
        \label{fig:kv_cache_usage:visual_tokens}
    \end{subfigure}
    \hfill
    \begin{subfigure}{0.245\linewidth}
        \centering
        \includegraphics[width=0.9\linewidth]{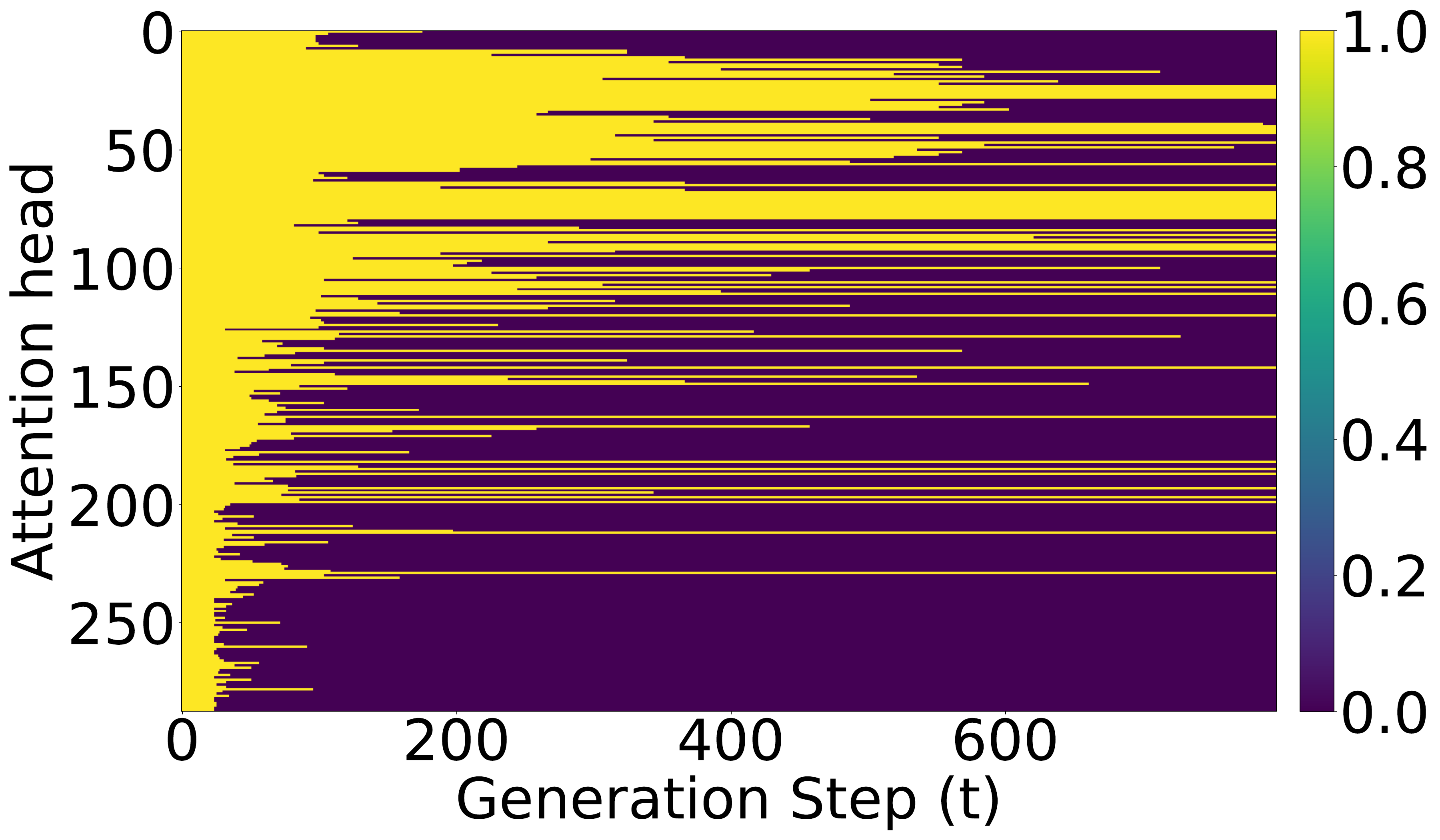}
        \caption{$\texttt{<|vision\_start|>}$}
        \label{fig:kv_cache_usage:vision_start}
    \end{subfigure}
    \begin{subfigure}{0.245\linewidth}
        \centering
        \includegraphics[width=0.9\linewidth]{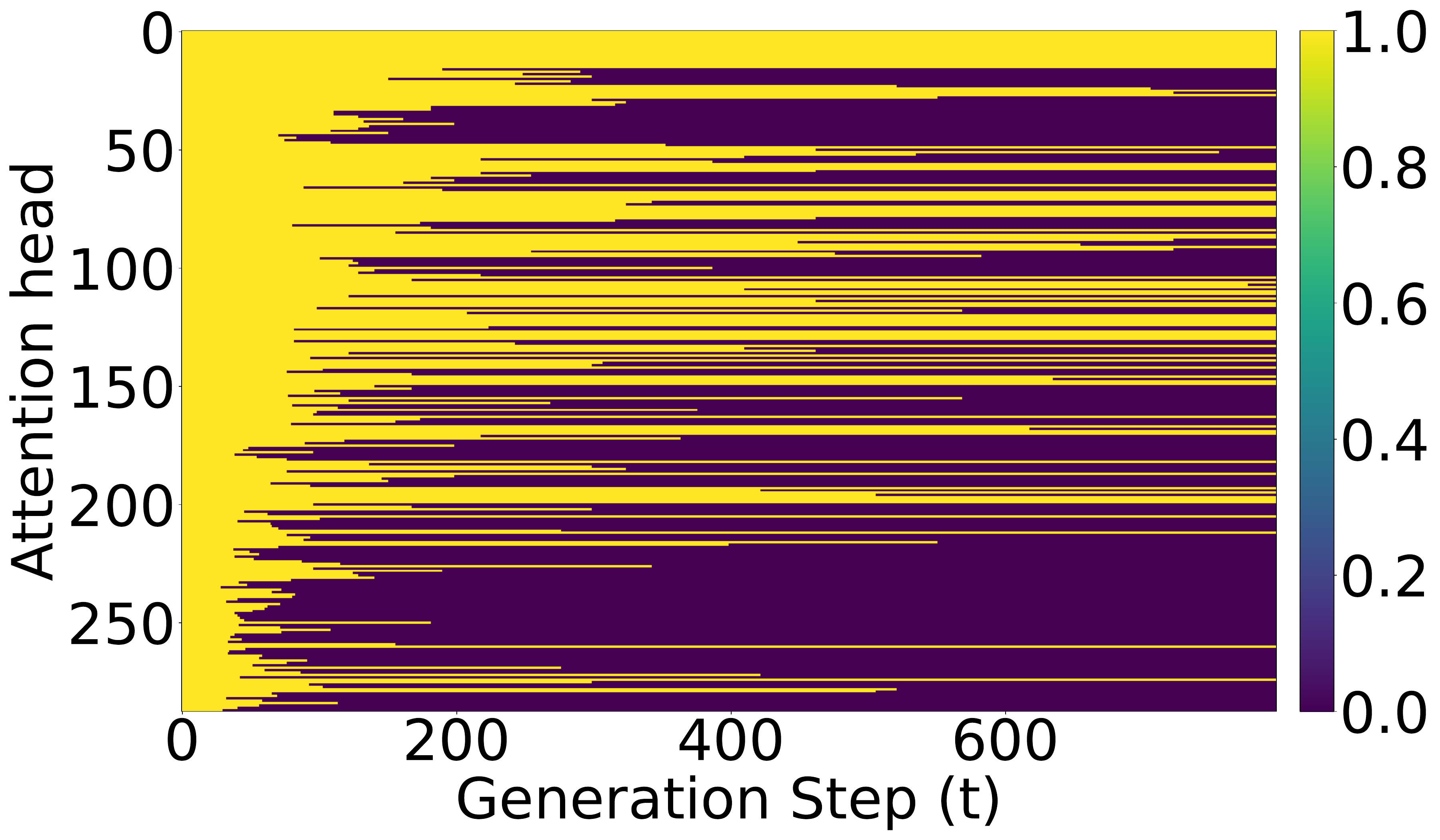}
        \caption{$\texttt{<|vision\_end|>}$}
        \label{fig:kv_cache_usage:vision_end}
    \end{subfigure}
    \caption{Visualization of KV-cache usage across attention heads for a sample in the MathVision\textsubscript{mini} dataset. The $y$-axis indexes the attention heads sequentially from layer 1 to layer 36, with each layer containing 8 KV heads. Figures~\ref{fig:kv_cache_usage:total_tokens},~\ref{fig:kv_cache_usage:visual_tokens},~\ref{fig:kv_cache_usage:vision_start}, and~\ref{fig:kv_cache_usage:vision_end} show the number of total tokens, visual tokens, $\texttt{<|vision\_start|>}$, and $\texttt{<|vision\_end|>}$ tokens that are active in the KV cache, respectively.}
    \label{fig:kv_cache_usage}
\end{figure*}

\textbf{Benchmarks.} We consider three image-based benchmarks, including MMStar~\citep{chen2024we}, MathVision\textsubscript{mini}~\citep{wang2024measuring}, and MMMUPro\textsubscript{vision}~\citep{yue2025mmmu}, and three video-based benchmarks, such as VideoMME~\citep{fu2025video}, VideoMathQA\textsubscript{mcq}~\citep{rasheed2025videomathqa}, and VideoMMMU~\citep{hu2025video} (adaptation and comprehension subsets). For a pure text reasoning, we evaluate on AIME24~\citep{aime2024}, GSM8K~\citep{cobbe2021training}, and MATH-500~\citep{hendrycks2021measuring}, similar to~\citep{bui2025cache}. Unless otherwise specified, we use greedy decoding and report accuracy as the evaluation metric. We adopt LMMs-Eval~\citep{zhang2024lmmseval} as the standard framework for evaluating different methods on VLM benchmarks. Results are reported using a global budget of $M \times (\#\text{layers}) \times (\#\text{heads})$ for a fair comparison with fixed-budget methods.

\textbf{Baselines.} We compare against strong KV-eviction methods for long-context decoding, including SnapKV~\citep{li2024snapkv}, AdaKV~\citep{feng2024ada}, Ada-Pyramid-KV~\citep{cai2024pyramidkv}, R-KV~\citep{cai2025r}, and TrimKV~\citep{bui2025cache}. Among these, AdaKV and Ada-Pyramid-KV extend SnapKV with dynamic budget allocation. For pure-text reasoning, we also include SeerAttn-R~\citep{gao2025seerattention}, a state-of-the-art learnable KV retrieval method tailored to reasoning-oriented language models.

\textbf{Implementation details.}
We use Qwen3-VL-8B-Thinking as a strong visual reasoning model. To train the retention gates, we use a mixture of multiple training datasets, including M4Instruct~\citep{li2024llava}, Academic Video~\citep{zhang2024videoinstructiontuningsynthetic}, R1-OneVision~\citep{yang2025r1}, and OpenR1-Math-220k~\citep{openr1math2025}, to cover diverse training scenarios. Since the retention gates are lightweight, comprising only 0.43\% of the total VLM parameters, we use a subset of each dataset for training. For the pure text experiment, we mimic the settings from~\citep{bui2025cache} and train on Qwen3-4B on OpenR1-Math-220K~\citep{openr1math2025} for fair comparisons. Other settings  are similar to Section~\ref{appendix:section:subsection:short-form_question_answering}.

\begin{wrapfigure}{r}{0.38\linewidth}
    \vspace{-5mm}
    \centering
    \includegraphics[width=\linewidth]{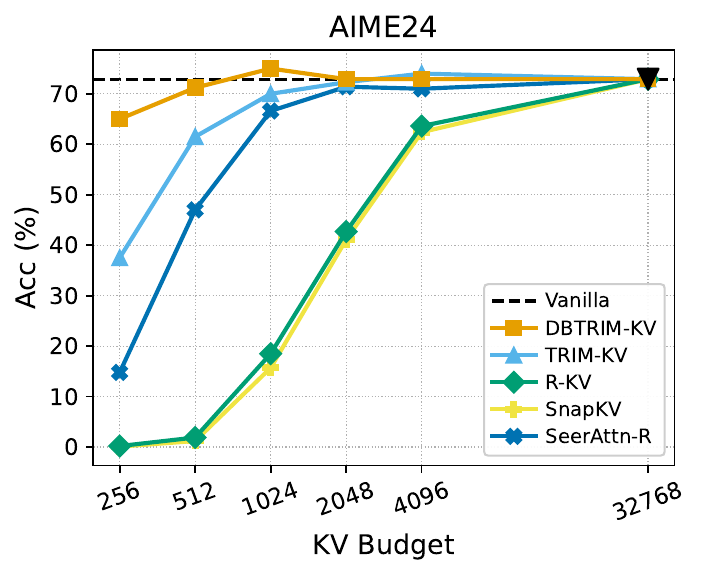}
    \caption{Long-reasoning performance of pure text Qwen3-4B. We defer the result for GSM8K and MATH-500 to Appendix~\ref{appendix:section:subsection:math_reasoning}}
    \vspace{-5mm}
    \label{fig:math}
\end{wrapfigure}
\textbf{Quantitative results.} Table~\ref{tab:main} reports the performance of different KV eviction methods on image and video reasoning benchmarks. Figure~\ref{fig:math} shows the Pareto fronts of different methods under varying KV budgets in pure text reasoning settings. It can be seen that DBTrimKV achieves state-of-the-art performance across the evaluated settings compared to other KV eviction baselines, with particularly large gains in low-budget regimes. In some settings, DBTrimKV surpasses full-cache inference at budgets of 256, 512, and 1024. 

Comparisons between fixed-budget and dynamic-budget methods (DBTrimKV vs.\ TrimKV and AdaKV/Ada-Pyramid-KV vs.\ SnapKV) show that dynamically allocating the KV budget generally yields better performance than using a fixed allocation, with the performance gap widening as the overall budget decreases. This indicates that, beyond improving memory and computational efficiency, token retention methods can also potentially enhance model performance relative to standard full-cache decoding. 

\textbf{Qualitative results.} Figure~\ref{fig:kv_cache_usage} illustrates DBTrimKV's highly dynamic and specialized KV cache allocation on a MathVision\textsubscript{mini} example. The total tokens (Figure ~\ref{fig:kv_cache_usage:total_tokens}) is aggressively compressed and concentrated into a sparse subset of heads, frequently in the mid-layers, suggesting the model selectively preserves context in the specific heads responsible for integrating semantic concepts and synthesizing information. For visual tokens (Figure ~\ref{fig:kv_cache_usage:visual_tokens}), dense context is retained early to interpret the image, but shrinks markedly as decoding shifts toward text-centric reasoning. Interestingly, while fine-grained visual patches are heavily evicted, the model persistently retains the boundary tokens $\texttt{<|vision\_start|>}$ (Figure ~\ref{fig:kv_cache_usage:vision_start}) and $\texttt{<|vision\_end|>}$ (Figure ~\ref{fig:kv_cache_usage:vision_end}) across numerous heads throughout generation. This suggests these tokens act as compressed structural anchors, grounding the model's reasoning without the memory cost of full visual details. Further visualizations demonstrating DBTrimKV's ability to isolate task-relevant visual tokens while aggressively discarding background distractors are provided in Appendix~\ref{appendix:section:visualization}.

\subsection{Multi-turn Dialogue}

\begin{table*}[t]
  \centering
  \resizebox{\textwidth}{!}{
    \begin{tabular}{l ccc ccc ccc}
      \toprule
      \multirow{2}{*}{\textbf{Method}} & \multicolumn{3}{c}{\textbf{512 Budget}} & \multicolumn{3}{c}{\textbf{256 Budget}} & \multicolumn{3}{c}{\textbf{128 Budget}} \\
      \cmidrule(lr){2-4} \cmidrule(lr){5-7} \cmidrule(lr){8-10}
      & \textbf{Accuracy} & \textbf{Overall} & \textbf{\% vs Vanilla} & \textbf{Accuracy} & \textbf{Overall} & \textbf{\% vs Vanilla} & \textbf{Accuracy} & \textbf{Overall} & \textbf{\% vs Vanilla} \\
      \midrule
      \rowcolor{gray!20}
      Vanilla & 3.37 & 3.57 & 100 & 3.37 & 3.57 & 100 & 3.37 & 3.57 & 100 \\
      \midrule
      SnapKV & 3.07 & 3.46 & 96.70 & 2.54 & 2.88 & 80.51 & 2.01 & 2.29 & 64.00 \\
      R-KV & 2.76 & 3.00 & 84.00 & 2.42 & 2.66 & 74.45 & 2.02 & 2.17 & 60.77 \\
      AdaKV & 2.89 & 3.24 & 90.79 & 2.56 & 2.88 & 80.68 & 2.03 & 2.25 & 62.98 \\
      \trimkv & \textcolor{TinaCrimson}{\underline{3.46}} & \textcolor{TinaCrimson}{\underline{3.73}} & \textcolor{TinaCrimson}{\textbf{104.51}} & \underline{3.16} & \underline{3.56} & 99.65 & \underline{3.05} & \underline{3.35} & 93.61 \\
      \rowcolor{\ourclr!20}
      \dbtrimkv & \textcolor{TinaCrimson}{\textbf{3.79}} & \textcolor{TinaCrimson}{\textbf{4.09}} & \textcolor{TinaCrimson}{\textbf{114.46}} & \textcolor{TinaCrimson}{\textbf{3.64}} & \textcolor{TinaCrimson}{\textbf{3.87}} & \textcolor{TinaCrimson}{\textbf{108.33}} & \textcolor{TinaCrimson}{\textbf{3.43}} & \textcolor{TinaCrimson}{\textbf{3.72}} & \textcolor{TinaCrimson}{\textbf{104.10}} \\
      \bottomrule
    \end{tabular}
  }
  \caption{Comparison of different methods across selected metrics under varying token budgets.}
  \label{tab:mmdu}
\end{table*}

We evaluate on multi-turn visual dialogue using MMDU~\citep{liu2024mmdu} to assess KV eviction under interactive, open-ended settings. Unlike single-shot QA, multi-turn dialogue requires the model to retain and selectively reuse information introduced many turns earlier, while accommodating new, potentially unexpected questions at each turn. In this experiment, we use Qwen3-VL-4B-Instruct as the base model and train on the MMDU-45K dataset~\citep{liu2024mmdu}. The MMDU-45K training data and MMDU evaluation benchmark are disjoint, with no evaluation examples used in training. For evaluation, we adopt Gemini 2.5 Flash~\citep{team2023gemini, comanici2025gemini} as an LLM-based judge, following the MMDU protocol. The judge receives only the dialogue and candidate responses, without method identities or method-dependent ordering. Other settings are similar to~Section~\ref{sec:exp_long}.

Table~\ref{tab:mmdu} shows that DBTrimKV and TrimKV outperform heuristic eviction baselines by a wide margin, with the advantage widening as the KV budget shrinks. Notably, DBTrimKV at all budgets outperforms vanilla inference with full cache (by 14\% at $M=512$). These results underscore the value of dynamic allocation: DBTrimKV consistently improves over its fixed-budget counterpart (TrimKV), and AdaKV similarly improves over SnapKV, indicating that adaptive capacity allocation is especially crucial for long, interactive dialogue applications.

\subsection{Runtime and Ablations}

\begin{wrapfigure}{r}{0.31\textwidth}
    \vspace{-4mm}
    \centering
    % --- Top Tables ---
    \resizebox{\linewidth}{!}{
    \begin{tabular}{lcc}
        \toprule
        Method/Budget & 256 & 512 \\
        \midrule
        Vanilla &  \multicolumn{2}{c}{48.68} \\
        \trimkv & 40.79 & 42.14 \\
        \dbtrimkv~w/o tying & 40.13 & 44.74 \\
        \dbtrimkv & 51.64 & 51.67 \\
        \bottomrule
    \end{tabular}
    }
    \captionof{table}{Weight tying ablation.}
    \label{tab:abl_weight_tying}
    \vspace{-6mm}
\end{wrapfigure}
We conduct ablation studies on the MathVision dataset to isolate the impact of our gate design.

\textbf{No weight tying.} 
Table~\ref{tab:abl_weight_tying} removes weight tying in the final linear layer of the retention gates. Weight tying yields a substantial improvement, especially in the low-budget regime, consistent with its role in encouraging score comparability across layers and heads.

\begin{wrapfigure}{r}{0.35\textwidth}
    \vspace{-4mm}
    \centering
    \resizebox{\linewidth}{!}{
    \begin{tabular}{lcc}
        \toprule
        Input/Budget & 256 & 1024 \\
        \midrule
        \dbtrimkv - $g_{\ell, h} ([\ve k_t || \ve v_t])$ & 43.75 & 46.71 \\
        \dbtrimkv - $g_{\ell, h} (\ve x_t)$ & 51.64 & 52.63 \\
        \bottomrule
    \end{tabular}
    }
    \captionof{table}{Input ablation.}
    \label{tab:abl_input}
    \vspace{-4mm}
\end{wrapfigure}
\textbf{Retention gate inputs.} We also test an alternative gate that takes the concatenation of the key and value vectors as input. As shown in Table~\ref{tab:abl_input}, using the token contextual embedding $\ve x_t$ as input performs better, particularly under tight budgets, suggesting that $\ve x_t$ provides a stronger long-horizon utility signal than the projected KV representations alone.

\begin{wrapfigure}{r}{0.36\textwidth}
    \vspace{-2mm}
    \centering
    \includegraphics[width=\linewidth]{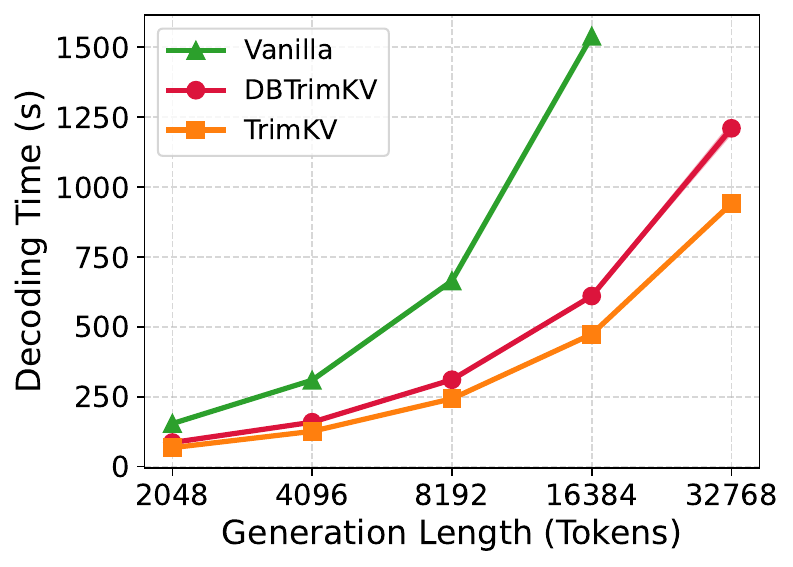}
    \caption{Efficiency scaling with generation length.}
    \label{fig:gen_length_main}
\end{wrapfigure}
\paragraph{Runtime.}
Figure~\ref{fig:gen_length_main} reports total decoding time in seconds as the number of generated tokens increases. Evaluated methods include Vanilla (Full Cache), DBTrimKV, and TrimKV. We fix batch size at 16, context length at 16384 and KV budget at 256. DBTrimKV with PagedAttention remains highly efficient. DBTrimKV adds a small overhead compared to TrimKV due to variable-length, head-specific cache handling; however, it still outperforms vanilla inference at scale. Overall, DBTrimKV is practical for large-batch deployment, with PagedAttention enabling efficient dynamic budgeting. We expect the speedups over vanilla to grow with longer contexts and longer generations, where KV bandwidth becomes the dominant bottleneck. More details are provided in Appendix~\ref{appendix:section:efficiency_scaling}.

\section{Related Work}

\textbf{KV cache compression.} Since the advent of long-context LMs and VLMs~\citep{liu2025comprehensive}, KV-cache compression has become a central efficiency challenge~\citep{li2024survey}. Early efforts primarily target the prefill stage and can be grouped into three main directions: (i) vector compression and quantization~\citep{hooper2024kvquant, liu2024kivi, yue2024wkvquant, sun2024shadowkv}, (ii) token retrieval~\citep{tang2024quest, liu2024retrievalattention, gao2025seerattention}, and (iii) token eviction or merging~\citep{zhang2023h2o, nawrot2024dynamic, zhang2024cam, qin2025cake, wang2025llms, liu2025can, park2025keydiff}. While effective in many settings, these approaches often incur nontrivial systems overhead (e.g., compression, selection, or data movement), making them less suitable for long-horizon generation tasks such as long reasoning, interactive use, or streaming applications. In these scenarios, simple token-eviction heuristics~\citep{xiao2023efficient, li2024snapkv, cai2025r} are typically preferred due to their low overhead, at the cost of potential performance degradation. Recent work~\citep{bui2025cache} proposes token retention as an effective learnable alternative for driving eviction policies. However, most studies on long-horizon efficiency focus on text-only LLMs, and the behavior and effectiveness of these methods in multimodal settings remain underexplored. We bridge this gap by providing an extensive evaluation of KV-cache compression strategies, including retention-based eviction, on VLM benchmarks.

\textbf{Dynamic budget allocation for KV cache.}
Beyond selecting \emph{which} tokens to evict, a complementary line of work studies how KV-cache \emph{capacity} should be allocated across layers, heads, or modalities. For text-only LMs, AdaKV~\citep{feng2024ada} adapts per-layer KV sizes using attention statistics, PyramidKV~\citep{cai2024pyramidkv} uses a hand-crafted pyramid that favors lower layers, CAKE~\citep{qin2025cake} allocates layer-wise budgets based on spatial and temporal attention dynamics, HeadKV~\citep{fu2025headkv} allocates head-wise budgets based on estimated retrieval and reasoning importance, and ZigZagKV~\citep{zhong2025zigzagkv} alternates dense and sparse layers to trade memory for accuracy. Dynamic budgeting has also been explored for VLMs: VL-KV~\citep{tu2024vl} introduces modality-aware KV compression, while MEDA~\citep{wan2025meda} allocates attention and memory differently across visual and textual tokens or visual frame as in~\citep{li2025less}. Despite improved utilization over uniform budgets, these methods typically rely on heuristics or separate, locally enforced budgets. In contrast, we eliminate per-layer/head/modality budgets and enforce a single \emph{global} KV constraint, using learned retention scores as the sole signal to jointly drive eviction and allocation.

\textbf{Visual token pruning.} In parallel with KV cache compression, a separate line of work focuses on visual token pruning. Recent approaches, such as VisionZip~\citep{yang2025visionzip} and FastV~\citep{chen2024image}, prune tokens using pretrained attention maps, while DivPrune~\citep{alvar2025divprune} and CDPruner~\citep{zhang2025beyond} enforce diversity within the selected vision tokens and SparseVLM~\citep{zhang2024sparsevlm} scores patches via text-to-vision attention. However, these methods operate primarily on vision tokens and often rely on unimodal or local attention signals for pruning. In contrast, our approach acts directly on the KV cache and jointly manages both vision and text tokens under a unified memory budget, allocating capacity across heads and layers based on downstream utility.

\section{Conclusion}

We proposed DBTrimKV for memory-efficient autoregressive decoding in long-context LMs and VLMs. DBTrimKV enforces a single global KV budget and evicts the tokens with the lowest retention scores across layers, heads, and modalities, yielding dynamic budget allocation without attention heuristics or hand-crafted per-head limits. Across short-form QA, long-horizon reasoning, and multi-turn dialogue, it achieves strong performance relative to prior eviction and dynamic-budget baselines at the same KV budget while maintaining performance close to full-cache inference with less memory. In some settings, DBTrimKV also exceeds full-cache performance, suggesting that selective eviction can offer benefits beyond memory reduction. These results suggest that learned token retention is an effective control signal for long-horizon VLM inference.

\section*{Acknowledgments}
This work was supported in part by the Yale AI Seed Grant Program (2026). We also thank Seoyoon Chae for pointing out a result aggregation bug with full-cache decoding in our Qwen3-4B results on AIME24.

% \begin{ack}
% Use unnumbered first level headings for the acknowledgments. All acknowledgments

% Do {\bf not} include this section in the anonymized submission, only in the final paper. You can use the \texttt{ack} environment provided in the style file to automatically hide this section in the anonymized submission.
% \end{ack}

\bibliography{main}
\bibliographystyle{plainnat}

%%%%%%%%%%%%%%%%%%%%%%%%%%%%%%%%%%%%%%%%%%%%%%%%%%%%%%%%%%%%

\appendix

\clearpage
\begin{center}
{\bf \Large{Appendix of ``Make Each Token Count: Towards Improving Long-Context Performance with KV Cache Eviction''}}
\end{center}

\DoToC

\clearpage

\section{Theoretical Results}\label{sec:proofs}

\subsection{Proofs for Attention Dilution}\label{sec:dilut_proof}

We now provide proofs for results in Section~\ref{sec:dilut_proof}.

\begin{proof}[Proof of Proposition~\ref{prop:tv_dilution}]
Let \(m=\max_{i\in U_t}z_{t,i}\). Then
\[
\sum_{i\in U_t}e^{z_{t,i}}\le |U_t|e^m,
\qquad
\sum_{j\in C_t\setminus U_t}e^{z_{t,j}}
\ge
\sum_{d\in D'_t}e^{z_{t,d}}
\ge
|D'_t|e^{m-\Delta}.
\]
Hence
\[
\sum_{i\in U_t}\alpha_{t,i}
=
\frac{\sum_{i\in U_t}e^{z_{t,i}}}
{\sum_{j\in C_t}e^{z_{t,j}}}
\le
\frac{|U_t|e^m}
{|U_t|e^m+|D'_t|e^{m-\Delta}}
=
\frac{1}{1+e^{-\Delta}|D'_t|/|U_t|}.
\]
Therefore,
\[
\delta_t
=
1-\sum_{i\in U_t}\alpha_{t,i}
\ge
\frac{e^{-\Delta}|D'_t|/|U_t|}
{1+e^{-\Delta}|D'_t|/|U_t|}.
\]
If \(\Delta=O(1)\), \(|U_t|=O(1)\), and \(|D'_t|\to\infty\), then
\(e^{-\Delta}|D'_t|/|U_t|\to\infty\), so \(\delta_t\to 1\) and
\(\sum_{i\in U_t}\alpha_{t,i}=1-\delta_t\to 0\).
\end{proof}

\begin{proof}[Proof of Proposition~\ref{prop:retention_reduces_dilution}]
By definition,
\[
\frac{\sum_{i\notin U_t}e^{z_{t,i}}}
{\sum_{i\in U_t}e^{z_{t,i}}}
=
\frac{\delta_t}{1-\delta_t}.
\]
Using the definitions of \(\rho_U\) and \(\rho_D\),
\[
\delta_t^r
=
\frac{\rho_D\sum_{i\notin U_t}e^{z_{t,i}}}
{\rho_U\sum_{i\in U_t}e^{z_{t,i}}
+
\rho_D\sum_{i\notin U_t}e^{z_{t,i}}}.
\]
Dividing numerator and denominator by
\(\rho_U\sum_{i\in U_t}e^{z_{t,i}}\) gives
\[
\delta_t^r
=
\frac{(\rho_D/\rho_U)\delta_t}
{(1-\delta_t)+(\rho_D/\rho_U)\delta_t}.
\]
If \(\rho_D\le \rho_U\), then \(\rho_D/\rho_U\le 1\), and since
\(a\mapsto a\delta_t/((1-\delta_t)+a\delta_t)\) is nondecreasing for
\(a\ge 0\),
\[
\delta_t^r
\le
\frac{\delta_t}{(1-\delta_t)+\delta_t}
=
\delta_t.
\]
Finally, if \(\rho_D/\rho_U\to 0\) and \(1-\delta_t>0\), then the displayed formula gives
\[
\delta_t^r\to 0.
\]
\end{proof}

\subsection{Geometric Retention as Query-Agnostic Future Utility}
\label{app:geom_retention}

We now justify the geometric retention form
\(r_{t,i}=\beta_i^{\,t-i}\) used by retention-gated attention. The central idea
is to view the usefulness of an old cached token as a persistence event. A token
should be retained if it is likely to remain useful, or close to useful, for
future query states. Under stable future query dynamics, this persistence
probability decays geometrically whenever the token has a uniformly positive
chance of leaving its useful region over a bounded time window.

We adopt a one-shot compression viewpoint. At decoding step \(t\), the existing
cache
\[
C_t=\{1,\dots,t\}
\]
is compressed to a smaller set of old tokens. Tokens generated after time \(t\)
are not part of this compression decision: they may be retained automatically or
handled by later compression rounds. Thus, when deciding whether an old token
\(i\le t\) should survive compression at time \(t\), we compare it only against
the other old cached tokens in \(C_t\).

For each old cached token \(i\le t\), define its old-cache future utility by
\begin{equation}
\label{eq:old_future_utility}
\bar G_i(t)
:=
\sum_{s=t+1}^T
w_{t,s}
\Pr(i\in U_s^{(t)}\mid \mathcal F_t),
\end{equation}
where \(w_{t,s}\ge 0\) are horizon weights and
\(U_s^{(t)}\subseteq C_t\) denotes the set of old cached tokens that
remain useful at future step \(s\). This quantity measures the expected future
benefit of retaining token \(i\) after compressing the old cache at time \(t\).

\begin{assumption}[Stable low-rank future query dynamics]
\label{asmp:query_markov}
For a fixed attention head and conditioning on \(\mathcal F_t\), suppose the
future query-side state evolves in a low-dimensional subspace:
\[
z_{s,i}
=
\mathbf x_s^\top(\mathbf W_Q^\top \mathbf W_K)\mathbf x_i
\approx
\mathbf r_s^\top \mathbf c_i+\varepsilon_{s,i},
\qquad
\mathbf r_s\in\mathbb R^m,\quad \mathbf c_i\in\mathbb R^m,
\]
where \(m\ll d\). The future process \((\mathbf r_s)_{s>t}\) is a stable,
time-homogeneous Markov chain on \(\mathbb R^m\) with transition kernel \(P\)
and invariant distribution \(\pi\). Stability means that the chain is
mean-reverting and ergodic, so that its distribution converges to \(\pi\) from
typical initial states.
\end{assumption}

\begin{figure}[H]
    \centering
    \includegraphics[width=\linewidth]{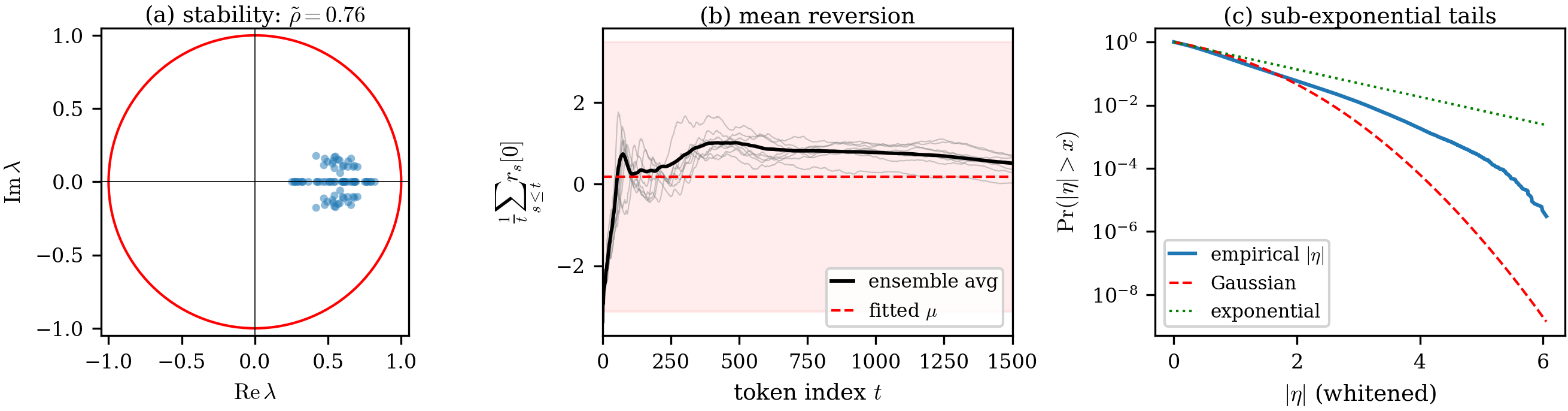}
    \caption{Empirical evidence for Assumption~\ref{asmp:query_markov}. We extract the per-token query state $\mathbf{q}_t$ at layer $20$, head $0$ of Qwen3-8B for $128$ rollouts each on $10$ AIME-24 prompts, project onto the top-$8$ PCA directions of the pooled training queries to obtain $\mathbf{r}_t\in\mathbb R^{8}$, and fit a per-context VAR(1) model $\mathbf{r}_{t+1}=A\mathbf{r}_t+\mathbf{b}+\boldsymbol{\eta}t$ on $75\%$ of rollouts. \textbf{(a) Stability.} Eigenvalues of $A$ across all $10$ contexts (each context contributes $8$ eigenvalues) lie strictly inside the unit circle; the median spectral radius is $\tilde\rho=0.76$ and every fitted context is stable, so the chain is mean-reverting. \textbf{(b) Mean reversion / ergodicity.} For one context, the per-rollout running mean $\frac{1}{t}\sum{s\leq t}r_s[0]$ (grey) and its ensemble average (black) converge to the fitted $\mu$ (red dashed) inside a $\pm\sigma$ stationary band; the time average matches the ensemble average, consistent with ergodicity. \textbf{(c) Sub-exponential innovations.} Pooled log-survival of the whitened innovation magnitude $|\eta|$ (blue) sits strictly between the Gaussian (red dashed) and exponential (green dotted) reference curves, satisfying the sub-exponential tail condition that drives the persistence bound.}
    \label{fig:placeholder}
\end{figure}

Under Assumption~\ref{asmp:query_markov}, future attention scores against old
cached tokens are determined, up to approximation error, by the future query
state \(\mathbf r_s\) and the fixed compatibility vectors
\(\{\mathbf c_j:j\le t\}\). Hence top-\(K\) membership among old cached tokens
can be described as a region in the future query-state space.

For a future query state \(\mathbf r\), define the hard old-cache top-\(K\)
threshold
\[
\Gamma_{K,t}(\mathbf r)
:=
K\text{-th largest value in }
\{\mathbf r^\top \mathbf c_j:j\in C_t\}.
\]
We say that token \(i\le t\) is immediately useful at query state \(\mathbf r\)
if it lies in the old-cache top-\(K\), namely if
\[
\mathbf r^\top \mathbf c_i
\ge
\Gamma_{K,t}(\mathbf r).
\]
Equivalently, the hard top-\(K\) useful region for token \(i\) is
\[
\mathcal T_{i,t}^{(K)}
:=
\left\{
\mathbf r\in\mathbb R^m:
\mathbf r^\top \mathbf c_i
\ge
\Gamma_{K,t}(\mathbf r)
\right\}.
\]
If \(\mathbf r\in\mathcal T_{i,t}^{(K)}\), then token \(i\) is among the
top-\(K\) old cached tokens for query state \(\mathbf r\).

However, immediate top-\(K\) membership is too strict as a retention criterion.
A token that is not currently in the top-\(K\) may still be worth retaining if
it remains close to the top-\(K\) boundary or if plausible future query states
can make it top-\(K\) again. Retention should therefore capture not only
instantaneous usefulness, but also future usefulness under small changes in the
query state.

To formalize this idea, we introduce a token-dependent relaxed threshold
\[
\widetilde{\Gamma}_{i,K,t}(\mathbf r)
\le
\Gamma_{K,t}(\mathbf r).
\]
Equivalently, we write
\[
\widetilde{\Gamma}_{i,K,t}(\mathbf r)
=
\Gamma_{K,t}(\mathbf r)
-
\Delta_{i,t}(\mathbf r),
\qquad
\Delta_{i,t}(\mathbf r)\ge 0,
\]
where \(\Delta_{i,t}\) is a token-dependent slack. The hard top-\(K\) rule is
recovered when \(\Delta_{i,t}\equiv 0\).

This slack has the following interpretation. Token \(i\) is useful in the hard
sense when it is in the old-cache top-\(K\). For retention, we lower the hard
threshold by the minimal amount needed to regard token \(i\) as still useful
along the future query trajectory we wish to preserve. Thus the relaxed
threshold can be viewed as the least permissive relaxation of the top-\(K\)
boundary under which token \(i\) remains retention-worthy.

Define the relaxed survival region
\[
\mathcal S_{i,t}^{(K)}
:=
\left\{
\mathbf r\in\mathbb R^m:
\mathbf r^\top \mathbf c_i
\ge
\widetilde{\Gamma}_{i,K,t}(\mathbf r)
\right\}.
\]
Since
\[
\widetilde{\Gamma}_{i,K,t}(\mathbf r)
\le
\Gamma_{K,t}(\mathbf r),
\]
we have
\[
\mathcal T_{i,t}^{(K)}
\subseteq
\mathcal S_{i,t}^{(K)}.
\]
Thus, every query state in which token \(i\) is hard top-\(K\) is also in its
relaxed survival region, while \(\mathcal S_{i,t}^{(K)}\) may additionally
include states where token \(i\) is not top-\(K\) but remains sufficiently
compatible with the query to justify retaining it. This definition separates immediate usefulness from
retention-worthiness: the hard top-\(K\) region identifies tokens that are
currently among the strongest old-cache competitors, while the relaxed survival
region identifies tokens that remain close enough to this condition to be worth
keeping. We thus have, for a fixed compression time \(t\) and a future step \(s>t\), define
\[
U_s^{(t)}
:=
\left\{
i\le t:
\mathbf r_u\in\mathcal S_{i,t}^{(K)}
\ \text{for all }u=t+1,\dots,s
\right\}.
\]

This persistence formulation is natural for monotone KV eviction. Once an old
token is removed during compression, it cannot re-enter the cache later. Hence a
retained token should not merely be useful for one future query; it should remain
within a region where it is likely to be useful over the intervening future query
trajectory.

\begin{figure}[t]
    \centering
    \includegraphics[width=0.7\linewidth]{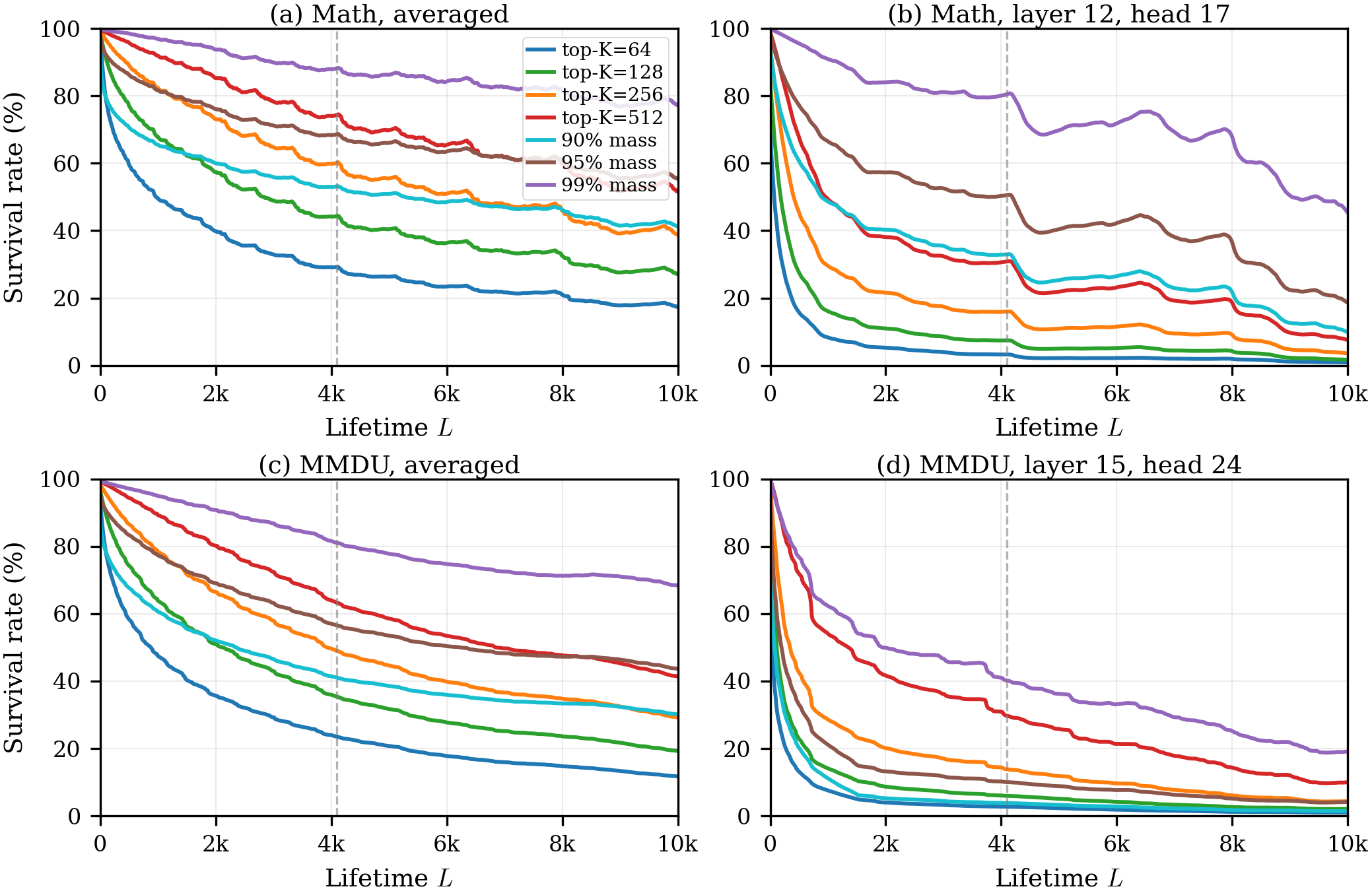}
    \caption{Token attention-survival under full-cache inference. \textbf{Top row}: Qwen3-4B on $100$ OpenR1-Math prefills up to $T = 16{,}384$ tokens. \textbf{Bottom row}: Qwen3-VL-4B on $98$ MMDU multimodal dialogues (text $+$ interleaved images), prefilled to the same length. \textbf{Left column}: survival rate averaged across all layers and query heads. \textbf{Right column}: survival rate for a single representative head (Math: layer $12$, head $17$; MMDU: layer $15$, head $24$). For each query position $p$ we identify the head's selected set of past tokens using either the top-$K$ logits ($K\in{64,128,256,512}$) or the smallest set whose softmax mass reaches a threshold $\tau\in{0.90, 0.95, 0.99}$; a token of birth time $i$ is then alive at horizon $L$ if it appears in the selected set of some query at position $p \ge i + L$. The dashed vertical line marks $L = 4096$. In every panel, every fixed top-$K$ curve drops sharply with $L$, while the more lenient mass criteria stay substantially higher: the gap between top-$K$ and $\tau$-mass curves is the attention-dilution mass — tokens that still carry meaningful aggregate weight but are pushed out of the head's hard top-$K$ budget. Heterogeneity at the head level is large: a representative head on either dataset loses the majority of tokens by $L = 4096$ even under the most lenient criterion ($99\%$ mass), in contrast to the across-head average. The same qualitative shape appears in the multimodal setting (rows c, d), confirming that attention dilution is not specific to text-only inputs.}
    \label{fig:app_survival}
\end{figure}

In autoregressive generation, token relevance is typically local and
query-dependent. Many cached tokens are useful only for a limited phase of the
generation: they may be highly compatible with the current query, but as the
query state drifts toward later semantic or syntactic contexts, they naturally
leave the region in which they are useful (see Figure~\ref{fig:app_survival}). This motivates the following block-exit
condition on the relaxed survival region. 

\begin{assumption}[Top-\(K\) block exit]
\label{asmp:relaxed_topk_block_exit}
Fix a compression time \(t\). For each old cached token \(i\le t\), there exist
an integer \(b_i\ge 1\) and a constant \(\epsilon_i>0\) such that
\[
\sup_{\mathbf r\in \mathcal S_{i,t}^{(K)}}
\Pr_{\mathbf r}\!\left(
\mathbf r_1,\dots,\mathbf r_{b_i}\in \mathcal S_{i,t}^{(K)}
\right)
\le
1-\epsilon_i .
\]
Here \(\Pr_{\mathbf r}\) denotes the law of the Markov chain initialized at
\(\mathbf r_0=\mathbf r\). Equivalently, whenever token \(i\) is currently
inside its relaxed survival region, there is probability at least
\(\epsilon_i\) that it exits this region within the next \(b_i\) future query
steps.
\end{assumption}

The block length \(b_i\) is allowed to
be token dependent. For tokens whose relevance is tied to the current local
query context, \(b_i\) may be small. For globally useful tokens, such as sink
tokens, delimiter tokens, or other persistent structural tokens, \(b_i\) may be
much larger, reflecting the slower decay of their usefulness.

\begin{theorem}[Exponential relaxed top-\(K\) persistence]
\label{thm:geom_relaxed_topk}
Suppose Assumptions~\ref{asmp:query_markov} and
\ref{asmp:relaxed_topk_block_exit} hold, and define
\(U_s^{(t)}\) through relaxed survival regions as above. Then, for
each old cached token \(i\le t\), there exist constants \(A_i>0\) and
\(\theta_i>0\) such that, for all \(n\ge 1\),
\[
\Pr(i\in U_{t+n}^{(t)}\mid\mathcal F_t)
\le
A_i e^{-\theta_i n}.
\]
In particular, one may take
\[
\theta_i
=
-\frac{1}{b_i}\log(1-\epsilon_i),
\qquad
A_i
=
(1-\epsilon_i)^{-1}.
\]
Equivalently, defining
\[
\beta_i
:=
e^{-\theta_i}
=
(1-\epsilon_i)^{1/b_i}
\in(0,1),
\]
the old-cache future utility in Eq.~\eqref{eq:old_future_utility} satisfies
\[
\bar G_i(t)
\le
A_i
\sum_{s=t+1}^T
w_{t,s}
\beta_i^{\,s-t}.
\]
\end{theorem}

\begin{proof}
Let
\[
S_n
:=
\left\{
\mathbf r_{t+1},\dots,\mathbf r_{t+n}
\in
\mathcal S_{i,t}^{(K)}
\right\}
\]
be the event that token \(i\)'s future query trajectory remains inside its
relaxed survival region for \(n\) consecutive future steps. By the definition of
\(U_s^{(t)}\),
\[
\Pr(i\in U_{t+n}^{(t)}\mid\mathcal F_t)
\leq
\Pr(S_n\mid\mathcal F_t).
\]

Write \(n=q b_i+r\), where \(q=\lfloor n/b_i\rfloor\) and
\(0\le r<b_i\). Applying the Markov property at the beginning of each block and
using Assumption~\ref{asmp:relaxed_topk_block_exit} gives
\[
\Pr(S_n\mid\mathcal F_t)
\le
\Pr(S_{q b_i}\mid\mathcal F_t)
\le
(1-\epsilon_i)^q.
\]
Since \(q\ge n/b_i-1\),
\[
(1-\epsilon_i)^q
\le
(1-\epsilon_i)^{-1}
\exp\!\left(
-\frac{-\log(1-\epsilon_i)}{b_i}n
\right).
\]
Therefore,
\[
\Pr(S_n\mid\mathcal F_t)
\le
A_i e^{-\theta_i n},
\]
with
\[
A_i=(1-\epsilon_i)^{-1},
\qquad
\theta_i=-\frac{1}{b_i}\log(1-\epsilon_i).
\]
The bound on \(\bar G_i(t)\) follows by substituting this
persistence bound into Eq.~\eqref{eq:old_future_utility}.
\end{proof}

Theorem~\ref{thm:geom_relaxed_topk} gives a probabilistic interpretation of
geometric retention. The survival parameter
\[
\beta_i=(1-\epsilon_i)^{1/b_i}
\]
summarizes how persistent token \(i\)'s future usefulness is expected to be.
Tokens that quickly exit their relaxed top-\(K\) survival regions have smaller
\(\beta_i\), while tokens that remain compatible with many future query states
have \(\beta_i\) closer to one.

In practice, the block-exit parameters \((b_i,\epsilon_i)\) are not estimated
explicitly. Instead, the retention gate learns a direct prediction of
\(\beta_i\) from the token representation, or equivalently from features that
determine its compatibility vector \(\mathbf c_i\). Thus, \(\beta_i\) can be
viewed as a learned summary of the geometry of the relaxed top-\(K\) survival
region \(\mathcal S_{i,t}^{(K)}\) under the future query dynamics.

\section{Experiments}
\label{appendix:section:experiment}
\subsection{General Experiment Settings}
\label{appendix:section:subsection:additional_experiment_settings}
We train with a maximum sequence length of 32,768, using a learning rate of $2\times10^{-4}$, weight decay $10^{-6}$, and cosine scheduling. The per-GPU batch size is 1 with gradient accumulation set to 4; all other hyperparameters follow the default HuggingFace Trainer settings. For the training datasets, we randomly sample 40\% of M4Instruct~\citep{li2024llava}, 30\% of Academic Video~\citep{zhang2024videoinstructiontuningsynthetic}, 30\% of R1-OneVision~\citep{yang2025r1}, and 20\% of OpenR1-Math-220k~\citep{openr1math2025}.  For the Qwen3-VL-8B-Thinking experiment in Table~\ref{tab:main}, training takes approximately 12 hours on four H100 GPUs, or 48 H100 GPU-hours.

\subsection{Baselines}
\begin{itemize}
    \item SnapKV~\citep{li2024snapkv} is an attention-score-based eviction heuristic, and R-KV extends this approach to reasoning models by identifying redundant tokens via key-vector similarity. 
    \item AdaKV~\citep{feng2024ada} builds on SnapKV with dynamic, head-wise budget allocation driven by attention statistics. 
    \item R-KV~\citep{cai2025r} is a training-free compression method that optimizes the KV cache of reasoning models by jointly scoring tokens based on attention importance and semantic redundancy. 
    \item TrimKV is our extension of the token retention strategy from~\citep{bui2025cache} to vision-language models.
\end{itemize}

\subsection{Needle in a Haystack Experiment Settings}\label{appendix:section:subsection:needle_in_a_haystack}

To demonstrate how DBTrimKV mitigates attention dilution by evicting distractors, we follow the Needle-in-a-Haystack (NIAH) experimental setup~\citep{hsieh2024ruler}. Because most recent LLMs have achieved near-perfect accuracy on single-NIAH tasks, we increase the difficulty by combining MK-NIAH, MV-NIAH, and MQ-NIAH into a single setting. We generate a dataset of 5,117 training samples with varying context sizes (from 4k to 200k) and randomized numbers of keys, values, and queries. For testing, we hold out 200 samples where the numbers of keys, values, and queries are fixed to 8, 4, and 4, respectively. Using Qwen3-4B-Instruct-2507 (which supports a 256k context window) as our base model, we evaluate two trainable KV eviction methods: TrimKV and DBTrimKV. Similar to the main experiments, we freeze the base model and train only the retention head.

\section{Additional Experimental Results}

\subsection{Math Reasoning} \label{appendix:section:subsection:math_reasoning}

Here we provide complementary results on math reasoning for Qwen3-4B on GSM8K and MATH.
\begin{figure*}
    \centering
    \includegraphics[width=0.7\linewidth]{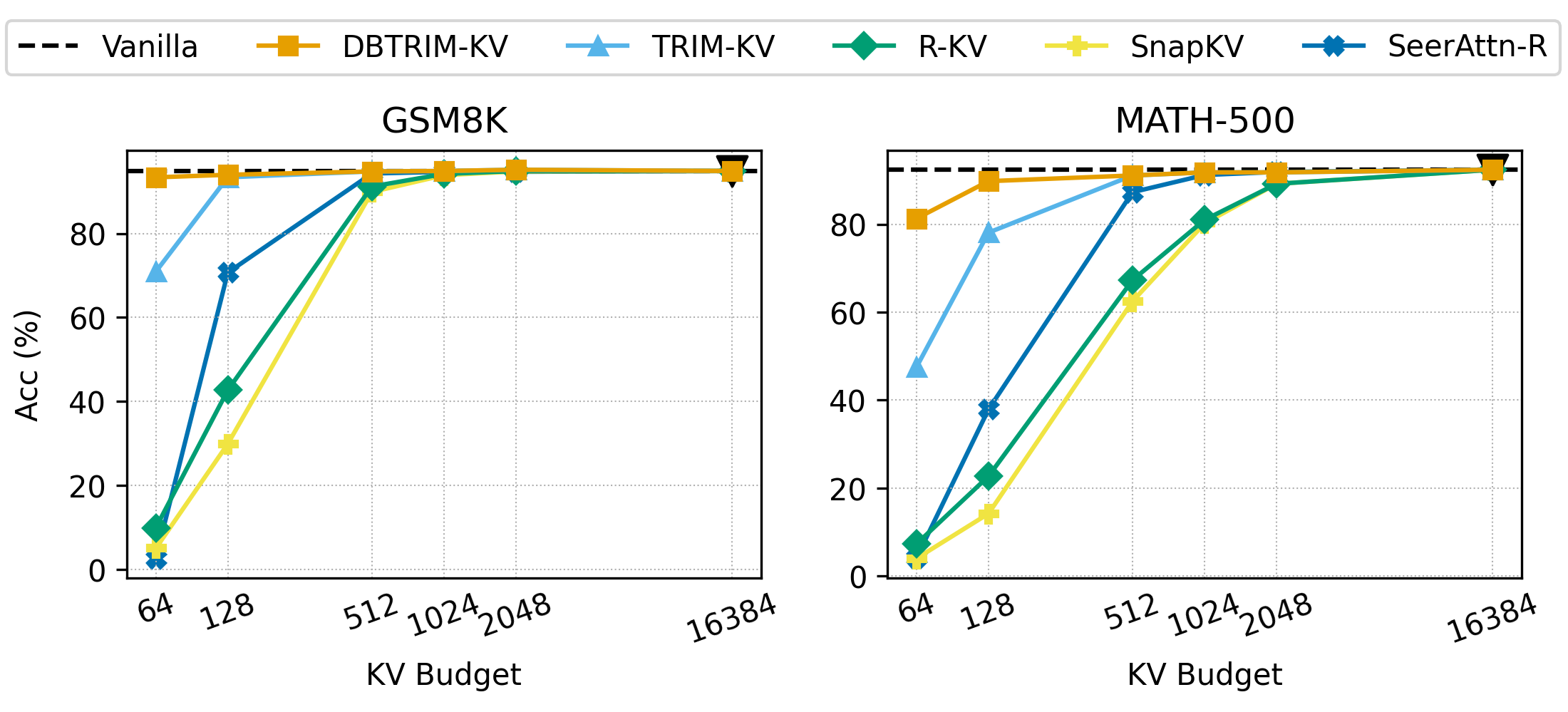}
    \caption{Qwen3-4B performance on GSM8K and MATH-500 across KV budgets.}
    \label{fig:math_full}
\end{figure*}

\subsection{MMDU}

Table~\ref{tab:complete_mmdu} reports the complete set of LLM-as-a-judge metrics in MMDU~\citep{liu2024mmdu}, serving as a detailed extension of Table~\ref{tab:mmdu}. 

The extended metrics reveal several critical trends regarding model performance under extreme sequence compression. First, DBTrimKV consistently outperforms not only the competing eviction methods but also the Vanilla full-cache baseline across all tested KV budgets (512, 256, and 128). Notably, even under the extreme constraint of a 128-token budget, DBTrimKV maintains an overall performance of 104.10\% relative to the Vanilla baseline, whereas existing methods like SnapKV, R-KV, and AdaKV suffer severe degradation, dropping to 60--64\% relative performance.

Furthermore, the metric breakdown highlights the specific failure modes of traditional eviction strategies. As the budget decreases, baseline methods exhibit sharp declines in multimodal-specific metrics such as ``Visual Perception'' and ``Image Relative,'' indicating a failure to retain visually grounded context. In contrast, both TrimKV and DBTrimKV preserve these critical tokens, allowing DBTrimKV to actually exceed Vanilla performance in visual perception (3.54 vs. 3.40) and logical coherence (4.11 vs. 3.90) at the lowest 128 budget.

\begin{table*}[h!]
  \centering
  \resizebox{\textwidth}{!}{
    \begin{tabular}{llcccccccc}
      \toprule
      \textbf{Budget} & \textbf{Method} 
      & \textbf{Answer Accuracy} & \textbf{Creativity} & \textbf{Richness} 
      & \textbf{Visual Perception} & \textbf{Logical Coherence} & \textbf{Image Relative} & \textbf{Overall} & \textbf{\% vs Vanilla} \\
      \midrule

      % ===================== VANILLA =====================
      \rowcolor{gray!20} \cellcolor{white}
      & Vanilla 
      & 3.37 & 3.68 & 4.21 & 3.40 & 3.90 & 3.40 & 3.57 & 100 \\
      \midrule

      % ===================== 512 =====================
      \multirow{5}{*}{512}
        & SnapKV   & 3.07 & \textcolor{TinaCrimson}{\underline{4.13}} & \textcolor{TinaCrimson}{\textbf{4.74}} & 2.78 & 3.88 & 3.05 & 3.46 & 96.70 \\
        & R-KV     & 2.76 & 3.32 & 3.77 & 2.58 & 3.43 & 2.78 & 3.00 & 84.00 \\
        & AdaKV    & 2.89 & \textcolor{TinaCrimson}{3.79} & \textcolor{TinaCrimson}{4.32} & 2.68 & 3.68 & 2.93 & 3.24 & 90.79 \\
        \cellcolor{white}
        & \trimkv   & \textcolor{TinaCrimson}{\underline{3.46}} & \textcolor{TinaCrimson}{3.88} & \textcolor{TinaCrimson}{4.25} & \textcolor{TinaCrimson}{\underline{3.51}} & \textcolor{TinaCrimson}{\underline{4.00}} & \textcolor{TinaCrimson}{\underline{3.64}} & \textcolor{TinaCrimson}{\underline{3.73}} & \textcolor{TinaCrimson}{\textbf{104.51}} \\
      \rowcolor{\ourclr!20} \cellcolor{white}
        & \dbtrimkv & \textcolor{TinaCrimson}{\textbf{3.79}} & \textcolor{TinaCrimson}{\textbf{4.25}} & \textcolor{TinaCrimson}{\underline{4.69}} & \textcolor{TinaCrimson}{\textbf{3.88}} & \textcolor{TinaCrimson}{\textbf{4.38}} & \textcolor{TinaCrimson}{\textbf{3.91}} & \textcolor{TinaCrimson}{\textbf{4.09}} & \textcolor{TinaCrimson}{\textbf{114.46}} \\
      \midrule

      % ===================== 256 =====================
      \multirow{5}{*}{256}
        & SnapKV   & 2.54 & 3.44 & 4.20 & 2.23 & 3.32 & 2.40 & 2.88 & 80.51 \\
        & R-KV     & 2.42 & 2.99 & 3.58 & 2.22 & 3.10 & 2.36 & 2.66 & 74.45 \\
        & AdaKV    & 2.56 & 3.41 & 4.16 & 2.25 & 3.31 & 2.45 & 2.88 & 80.68 \\
         \cellcolor{white}
        & \trimkv   & \underline{3.16} & \textcolor{TinaCrimson}{\underline{3.98}} & \textcolor{TinaCrimson}{\underline{4.25}} & \underline{3.35} & \textcolor{TinaCrimson}{\underline{3.97}} & \textcolor{TinaCrimson}{\underline{3.63}} & \underline{3.56} & 99.65 \\
      \rowcolor{\ourclr!20} \cellcolor{white}
        & \dbtrimkv & \textcolor{TinaCrimson}{\textbf{3.64}} & \textcolor{TinaCrimson}{\textbf{4.09}} & \textcolor{TinaCrimson}{\textbf{4.54}} & \textcolor{TinaCrimson}{\textbf{3.74}} & \textcolor{TinaCrimson}{\textbf{4.18}} & \textcolor{TinaCrimson}{\textbf{3.73}} & \textcolor{TinaCrimson}{\textbf{3.87}} & \textcolor{TinaCrimson}{\textbf{108.33}} \\
      \midrule

      % ===================== 128 =====================
      \multirow{5}{*}{128}
        & SnapKV   & 2.01 & 2.59 & 3.31 & 1.72 & 2.85 & 1.85 & 2.29 & 64.00 \\
        & R-KV     & 2.02 & 2.21 & 2.63 & 1.93 & 2.71 & 2.06 & 2.17 & 60.77 \\
        & AdaKV    & 2.03 & 2.54 & 3.21 & 1.78 & 2.83 & 1.92 & 2.25 & 62.98 \\
       \cellcolor{white}
        & \trimkv   & \underline{3.05} & \textcolor{TinaCrimson}{\underline{3.85}} & \underline{4.11} & \underline{3.03} & \underline{3.88} & \textcolor{TinaCrimson}{\underline{3.41}} & \underline{3.35} & 93.61 \\
      \rowcolor{\ourclr!20} \cellcolor{white}
        & \dbtrimkv & \textcolor{TinaCrimson}{\textbf{3.43}} & \textcolor{TinaCrimson}{\textbf{3.93}} & \textcolor{TinaCrimson}{\textbf{4.31}} & \textcolor{TinaCrimson}{\textbf{3.54}} & \textcolor{TinaCrimson}{\textbf{4.11}} & \textcolor{TinaCrimson}{\textbf{3.57}} & \textcolor{TinaCrimson}{\textbf{3.72}} & \textcolor{TinaCrimson}{\textbf{104.10}} \\
      \bottomrule
    \end{tabular}
  }
  \caption{Comparison of different methods across various metrics under varying token budgets. The best eviction methods are shown in \textbf{bold}, and the second-best are \underline{underlined}. Performances that exceed vanilla inference are highlighted in \textcolor{TinaCrimson}{\textbf{red}}. The superscript \textsuperscript{\orgfire} denotes a trainable KV eviction method.}
  \label{tab:complete_mmdu}
\end{table*}

\subsection{Additional Comparisons with CAKE and HeadKV}
\label{appendix:additional_baselines}

We compare DBTrimKV with CAKE~\citep{qin2025cake} and HeadKV~\citep{fu2025headkv}, which allocate cache capacity across layers and heads, respectively. We evaluate on MathVision\textsubscript{mini} (304 problems) using Qwen3-VL-8B-Thinking, with a KV budget of $M=1024$ per head on average (a total budget of $M\times L\times H$), a maximum of 32,768 generated tokens, and batch size 1. Table~\ref{tab:additional_baselines} shows that DBTrimKV achieves higher accuracy than both methods at this budget, supporting the effectiveness of joint token selection and global budget allocation.

\begin{table}[H]
    \centering
    \begin{tabular}{lc}
        \toprule
        Method & Accuracy (\%) \\
        \midrule
        Full cache & 48.68 \\
        SnapKV & 15.13 \\
        R-KV & 24.67 \\
        Ada-Pyramid-KV & 28.29 \\
        AdaKV & 32.89 \\
        CAKE & 20.07 \\
        HeadKV & 35.20 \\
        TrimKV & 45.72 \\
        DBTrimKV & \textbf{52.63} \\
        \bottomrule
    \end{tabular}
    \caption{Additional baseline comparisons on MathVision\textsubscript{mini} with Qwen3-VL-8B-Thinking at budget $M=1024$. The full-cache score is the original result reported in Table~\ref{tab:main}.}
    \label{tab:additional_baselines}
\end{table}

\subsection{Out-of-Distribution Training}
\label{appendix:ood_training}

To evaluate transfer across training domains, we train the Qwen3-4B retention gates on a general-QA mixture of Synth-32K, BookSum, and Buddhi, following the general-data setting of~\citet{bui2025cache}, without math or reasoning-mode supervision. We then evaluate on AIME24. Table~\ref{tab:ood_training} compares these gates with gates trained on OpenR1-Math and with existing baselines at the same cache budgets.

The general-QA-trained gates outperform SnapKV and R-KV at all four evaluated budgets, but underperform the math-trained gates. The gap between the two DBTrimKV variants is smaller at larger budgets. These results show transfer to an unseen task domain while also indicating that in-domain supervision remains beneficial, particularly under tight cache budgets.

\begin{table}[H]
    \centering
    \small
    \begin{tabular}{lrrrrr}
        \toprule
        Method / KV budget & 256 & 512 & 1024 & 2048 & 4096 \\
        \midrule
        SnapKV & -- & 1.2 & 15.7 & 41.1 & 62.4 \\
        R-KV & -- & 1.9 & 18.5 & 42.7 & 63.6 \\
        SeerAttention-R (OpenR1-Math) & -- & 47.0 & 66.6 & 71.4 & 71.0 \\
        DBTrimKV (OpenR1-Math) & \textbf{65.0} & \textbf{71.2} & \textbf{75.0} & \textbf{72.9} & \textbf{72.9} \\
        DBTrimKV (general QA) & -- & 36.7 & 56.2 & 63.3 & 68.8 \\
        \bottomrule
    \end{tabular}
    \caption{AIME24 accuracy (\%) with Qwen3-4B under different training domains and KV budgets. Training data are indicated in parentheses for the learned methods. Dashes indicate unreported results.}
    \label{tab:ood_training}
\end{table}

\subsection{Long Context Benchmarks}
\label{appendix:section:subsection:long_context_benchmarks}

\subsubsection{LongBench-V2}
In this section, we compare our method against Full KV Cache, TrimKV, and Locret~\citep{huang2025locret} on LongBenchv2 benchmark~\citep{bai2024longbench2}. We follow the chunk-prefill settings from~\citep{huang2025locret}. Base model is \texttt{microsoft/Phi-3-mini-128k-instruct}, locret model is from \texttt{hyx21/Locret-phi-3-mini-128K}. LongBench-v2 evaluates deep reasoning across realistic long-context scenarios, such as multi-document QA and code repository comprehension. Locret is a lightweight, training-based eviction method that employs learnable retaining heads to predict the causal importance of tokens. As shown in Table~\ref{tab:longbenchv2}, DBTrimKV achieves 9.20\% average accuracy improvement over the Full KV baseline, whereas both Locret and TrimKV experience performance degradation.

\begin{table}[h!]
\centering
\scriptsize
% \resizebox{\linewidth}{!}{%
\begin{tabular}{lrrrrrr}
\toprule
Method & Acc. Short & Acc. Easy & Acc. Medium & Acc. Hard & Avg. Acc & Avg. $\Delta$ (\%) \\
\midrule
\rowcolor{gray!20}
Full KV & \underline{33.71} & \underline{34.44} & 18.60 & 25.86 & \underline{28.79} & \underline{0.00} \\
LocRet & 30.34 & 30.00 & \textcolor{TinaCrimson}{\textbf{24.42}} & \textcolor{TinaCrimson}{\underline{27.59}} & 28.41 & -1.32 \\
TRIM-KV & 29.21 & 31.11 & \textcolor{TinaCrimson}{\underline{19.77}} & 23.56 & 26.13 & -9.24 \\
\rowcolor{\ourclr!20}
DBTrimKV & \textcolor{TinaCrimson}{\textbf{37.08}} & \textcolor{TinaCrimson}{\textbf{36.67}} & \textcolor{TinaCrimson}{\underline{19.77}} & \textcolor{TinaCrimson}{\textbf{28.74}} & \textcolor{TinaCrimson}{\textbf{31.44}} & \textcolor{TinaCrimson}{\textbf{+9.20}} \\
\bottomrule
\end{tabular}
% }
\vspace{2mm}
\caption{Performance on long-context tasks of KV eviction methods with Phi3-mini-128K on the LongBench-V2 benchmark, including average relative change (Avg. $\Delta$) compared to Full KV. KV Budget is set to 1024.}
\label{tab:longbenchv2}
\end{table}

\subsection{Ablation Study on Lookahead Step (T-t)}
During inference, we use a default lookahead horizon $T-t =2$ (as described in Section~\ref{appendix:section:subsection:additional_experiment_settings}) for every eviction time $t$. We now ablate the effect of this lookahead parameter on the MathVision\textsubscript{mini} benchmark across various KV budgets. As shown in Table~\ref{tab:abl_look_ahead_step}, a lookahead step of $T-t=2$ generally yields the highest accuracy for moderate to high KV budgets (256, 512, and 1024). Reducing the lookahead to a myopic one-step horizon ($T-t=1$) consistently degrades performance across these budgets. Conversely, extending the lookahead horizon to $T-t=5$ harms performance at higher budgets, such as dropping from 52.63\% to 46.05\% at a budget of 1024. However, under extremely constrained memory budgets (64 and 128), the longer lookahead horizon of $T-t=5$ becomes beneficial, outperforming both $T-t=1$ and $T-t=2$. This indicates that a moderate lookahead optimally balances recent context with long-term retention under typical memory conditions, whereas a longer lookahead helps prioritize the most globally persistent tokens when cache capacity is severely limited.

\begin{table}[t]
    \centering
    \footnotesize
    \resizebox{0.55\linewidth}{!}{
    \begin{tabular}{lccccc}
        \toprule
        Lookahead/Budget & 64 & 128 & 256 & 512 & 1024 \\
        \midrule
        \dbtrimkv - LHS $T-t=1$ & 34.54 & 46.43 & 49.34 & 50.33 & 50.66 \\
        \dbtrimkv - LHS $T-t=2$ & 33.88 & 47.04 & \textbf{51.64} & \textbf{51.97} & \textbf{52.63} \\
        \dbtrimkv - LHS $T-t=5$ & \textbf{38.49} & \textbf{47.73} & 47.70 & 50.00 & 46.05 \\
        \bottomrule
    \end{tabular}
    }
    \caption{Lookahead-step $T-t$ ablation on MathVision\textsubscript{mini}. Here, $t$ is the compression time.}
    \label{tab:abl_look_ahead_step}
\end{table}
\subsection{Efficiency Scaling} \label{appendix:section:efficiency_scaling}
To evaluate inference efficiency, we benchmark in a synthetic setting using dummy input sequences of varying context lengths, populated with filler tokens and an all-ones attention mask. We compare three methods: Vanilla (Full Cache), our proposed DBTrimKV, and TrimKV. Efficiency is evaluated by measuring total decoding time (seconds) and throughput (generated tokens per second). Timing is recorded during the generation phase and averaged over multiple steps following an initial warmup period. Results are ported in Figure~\ref{fig:context_length},~\ref{fig:gen_length},~\ref{fig:scaling_kv_budget}.

Across all scaling dimensions, our method, DBTrimKV, demonstrates significant efficiency improvements over the Vanilla baseline. As context length grows, the Vanilla model suffers from linearly increasing decoding times and decreasing throughput, whereas DBTrimKV maintains a consistent, flat efficiency profile. When scaling generation length, Vanilla exhibits a severe, quadratic increase in decoding time, whereas DBTrimKV scales much more efficiently, maintaining significantly higher throughput and lower decoding times even at extended lengths. Finally, across varying KV budgets, DBTrimKV consistently retains its efficiency advantage over the static Full Cache baseline, with throughput and decoding time remaining largely stable regardless of the specific budget allocation.

% Figure 1: Context Length
\begin{figure}[htbp]
    \centering
    \begin{subfigure}[b]{0.48\textwidth}
        \centering
        \includegraphics[width=\textwidth]{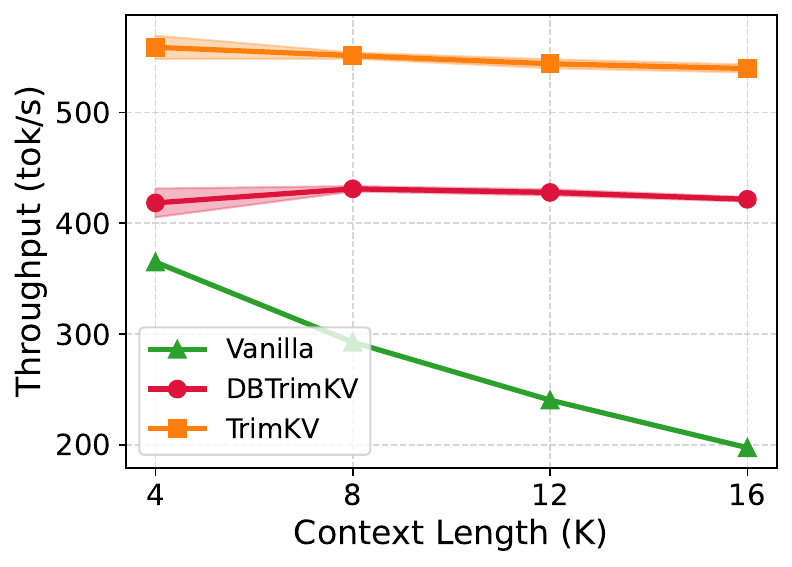}
        \label{fig:scaling_ctx_throughput}
    \end{subfigure}
    \hfill
    \begin{subfigure}[b]{0.48\textwidth}
        \centering
        \includegraphics[width=\textwidth]{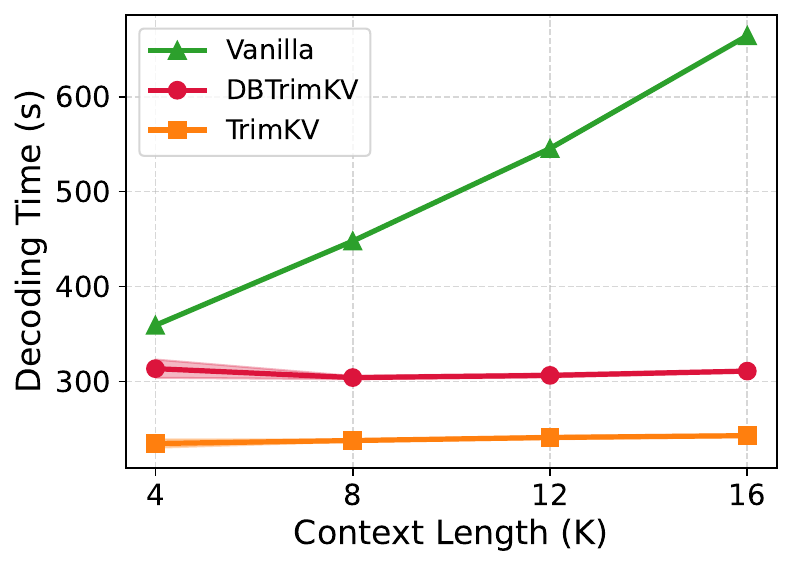}
        \label{fig:scaling_ctx_decoding}
    \end{subfigure}
    \caption{Efficiency scaling with context length. The figure reports throughput in tokens per second (left) and total decoding time in seconds (right) as the input context length increases. Evaluated methods include Vanilla (Full Cache), DBTrimKV, and TrimKV. We fix batch size at 16, number of generated tokens at 8192 and KV budget at 256. Results are averaged over 5 random seeds.} 
    \label{fig:context_length}
\end{figure}

% Figure 2: Generation Length
\begin{figure}[htbp]
    \centering
    \begin{subfigure}[b]{0.48\textwidth}
        \centering
        \includegraphics[width=\textwidth]{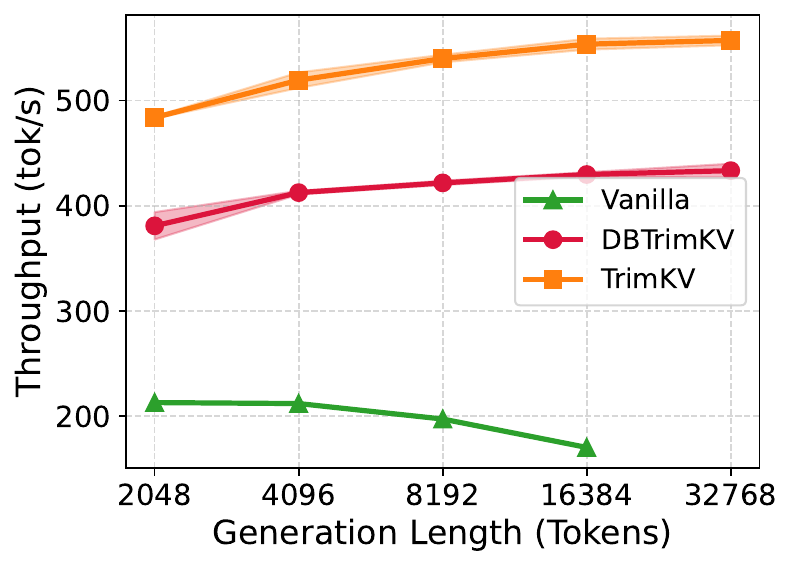}
        \label{fig:scaling_gen_throughput}
    \end{subfigure}
    \hfill
    \begin{subfigure}[b]{0.48\textwidth}
        \centering
        \includegraphics[width=\textwidth]{figures/gen_length_decoding_time.pdf}
        \label{fig:scaling_gen_decoding}
    \end{subfigure}
    \caption{Efficiency scaling with generation length. The figure reports throughput in tokens per second (left) and total decoding time in seconds (right) as the number of generated tokens increases. Evaluated methods include Vanilla (Full Cache), DBTrimKV, and TrimKV. We fix batch size at 16, context length at 16384 and KV budget at 256. We do not report vanilla performance for generation length at 32k due to OOM error. Results are averaged over 5 random seeds.} 
    \label{fig:gen_length}
\end{figure}

% Figure 3: KV Budget
\begin{figure}[htbp]
    \centering
    \begin{subfigure}[b]{0.48\textwidth}
        \centering
        \includegraphics[width=\textwidth]{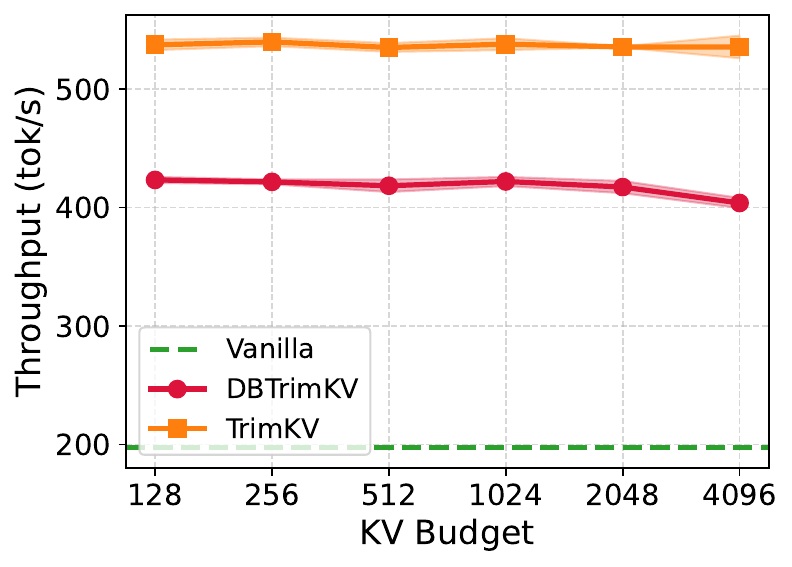}
        \label{fig:scaling_kv_throughput}
    \end{subfigure}
    \hfill
    \begin{subfigure}[b]{0.48\textwidth}
        \centering
        \includegraphics[width=\textwidth]{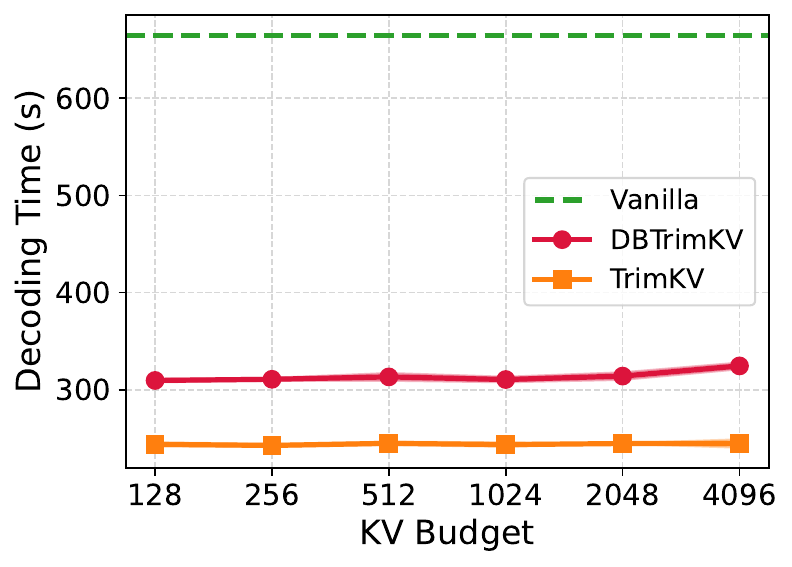}
        \label{fig:kv_decoding}
    \end{subfigure}
    \caption{Efficiency scaling with KV budget. The figure reports throughput in tokens per second (left) and total decoding time in seconds (right) across different allocated KV memory budgets. Evaluated methods include Vanilla (Full Cache), DBTrimKV, and TrimKV. We fix batch size at 16, number of generated tokens at 8192 and number of generated tokens at 8192. Results are averaged over 5 random seeds.} 
    \label{fig:scaling_kv_budget}
\end{figure}

\subsection{Visualization of DBTrimKV}
\label{appendix:section:visualization}
In this section, we visualize the visual tokens retained during inference to illustrate the model's selection behavior. As shown in Figures~\ref{fig:visual_token_visualization_mmmu_pro_test_2} and \ref{fig:visual_token_visualization_mmmu_pro_test_61}, DBTrimKV initially preserves a high density of visual tokens during the early stages of generation to establish structural and semantic understanding (e.g., Figure~\ref{fig:visual_token_visualization_mmmu_pro_test_2:step_15}). 

As decoding progresses, DBTrimKV aggressively evicts background distractors, narrowing its focus strictly to the most informative image regions necessary for the ongoing reasoning task (e.g., Figure~\ref{fig:visual_token_visualization_mmmu_pro_test_2:last_step}). Notably, the retained visual tokens align directly with the model's textual output. In Figure~\ref{fig:visual_token_visualization_mmmu_pro_test_2}, retention is heavily concentrated on the scythe, the clock, and the specific "RATHER TIME" text banner, which the model explicitly relies upon to deduce the final answer. Similarly, in Figure~\ref{fig:visual_token_visualization_mmmu_pro_test_61}, retention isolates the necessary coordinate data points. This dynamic focusing empirically validates that our learned retention mechanism successfully mitigates attention dilution by preserving only tokens with long-term utility while discarding irrelevant visual context. This trend also aligns with the layer-wise attention patterns visualized in Figure~\ref{fig:kv_cache_usage}, where a massive initial KV cache is gradually pruned as generation continues.

\begin{figure*}[htbp]
    \centering
    \begin{subfigure}[b]{0.48\textwidth}
        \centering
        \includegraphics[width=\textwidth]{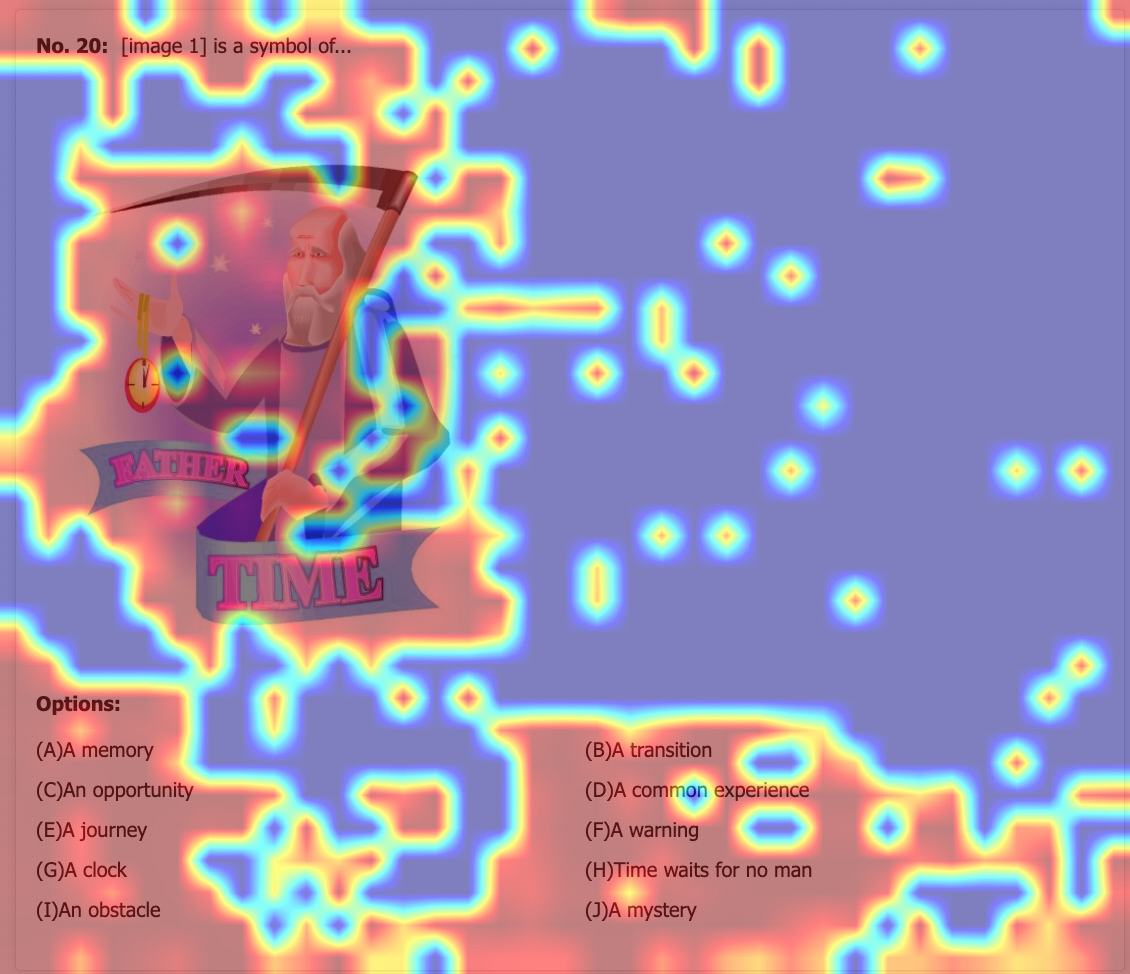}
        \caption{Generation step 15}
        \label{fig:visual_token_visualization_mmmu_pro_test_2:step_15}
    \end{subfigure}
    \hfill
    \begin{subfigure}[b]{0.48\textwidth}
        \centering
        \includegraphics[width=\textwidth]{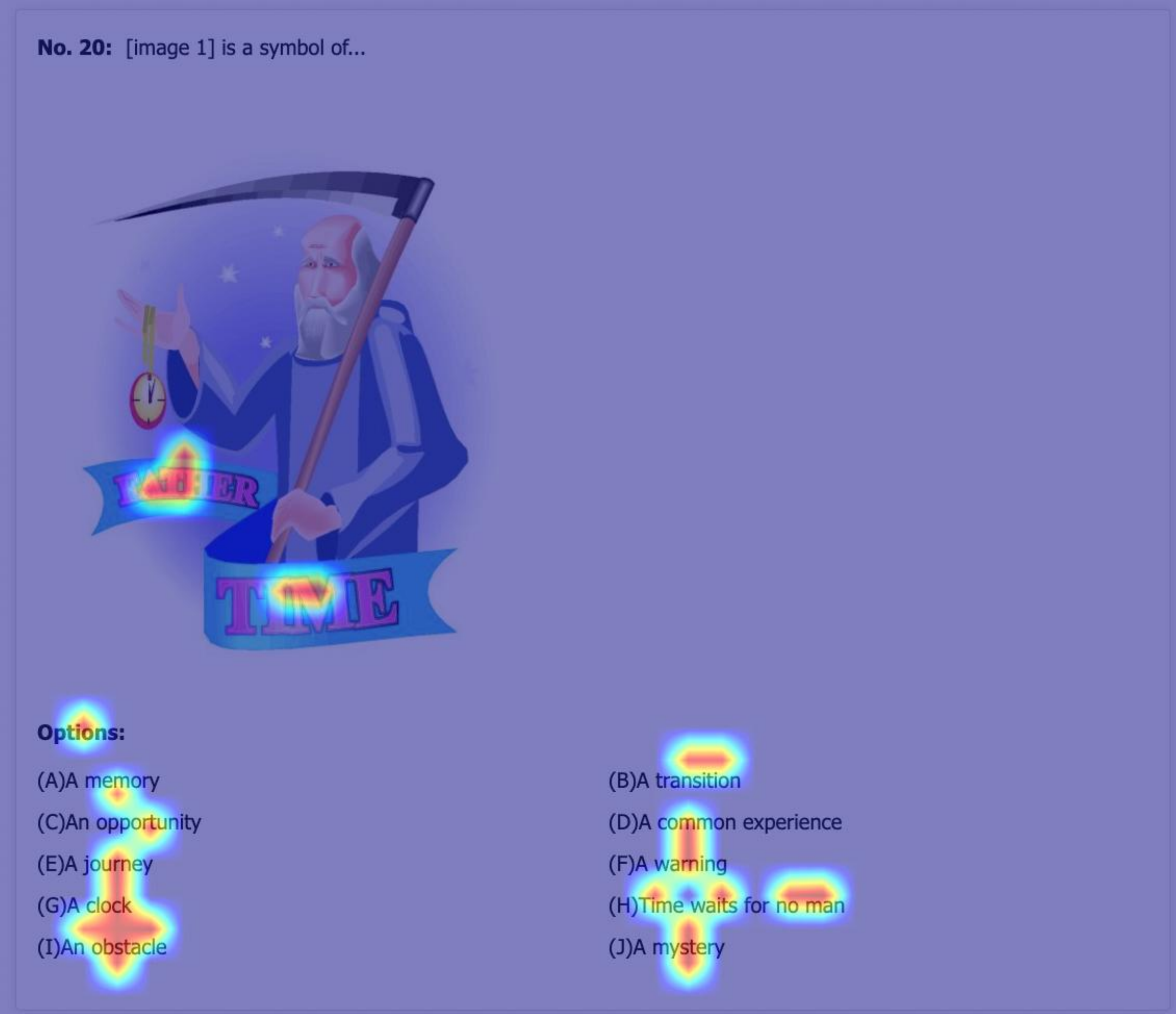}
        \caption{Last generation step}
        \label{fig:visual_token_visualization_mmmu_pro_test_2:last_step}
    \end{subfigure}
    
    \vspace{1em}
    
    \begin{subfigure}[b]{\textwidth}
        \centering
        \fbox{\parbox{0.95\textwidth}{\small
        let's look at the image. The character is holding a scythe (which is usually associated with Death) but also a clock, and the banner says "RATHER TIME". Wait, the options include (H) Time waits for no man. That's a common saying. Let me check the options again. The image has "RATHER TIME" which might be a play on "Time waits for no man". So the symbol here is representing that phrase. So the answer should be H.
        }}
        \caption{Model reasoning}
        \label{fig:visual_token_visualization_mmmu_pro_test_2:reasoning}
    \end{subfigure}
    \caption{Visualization of retained visual tokens for a sample in the MMMU-Vision-Pro dataset (Layer 13, Head 2). While the initial step (a) retains broad context, the final step (b) aggressively prunes background noise to isolate the scythe, the clock, and the ``RATHER TIME'' banner. This precise visual retention directly grounds the model's textual deduction (c) that the symbol represents the phrase ``Time waits for no man'' (Option H).}
    \label{fig:visual_token_visualization_mmmu_pro_test_2}
\end{figure*}

\begin{figure*}[htbp]
    \centering
    \begin{subfigure}[b]{0.48\textwidth}
        \centering
        \includegraphics[width=\textwidth]{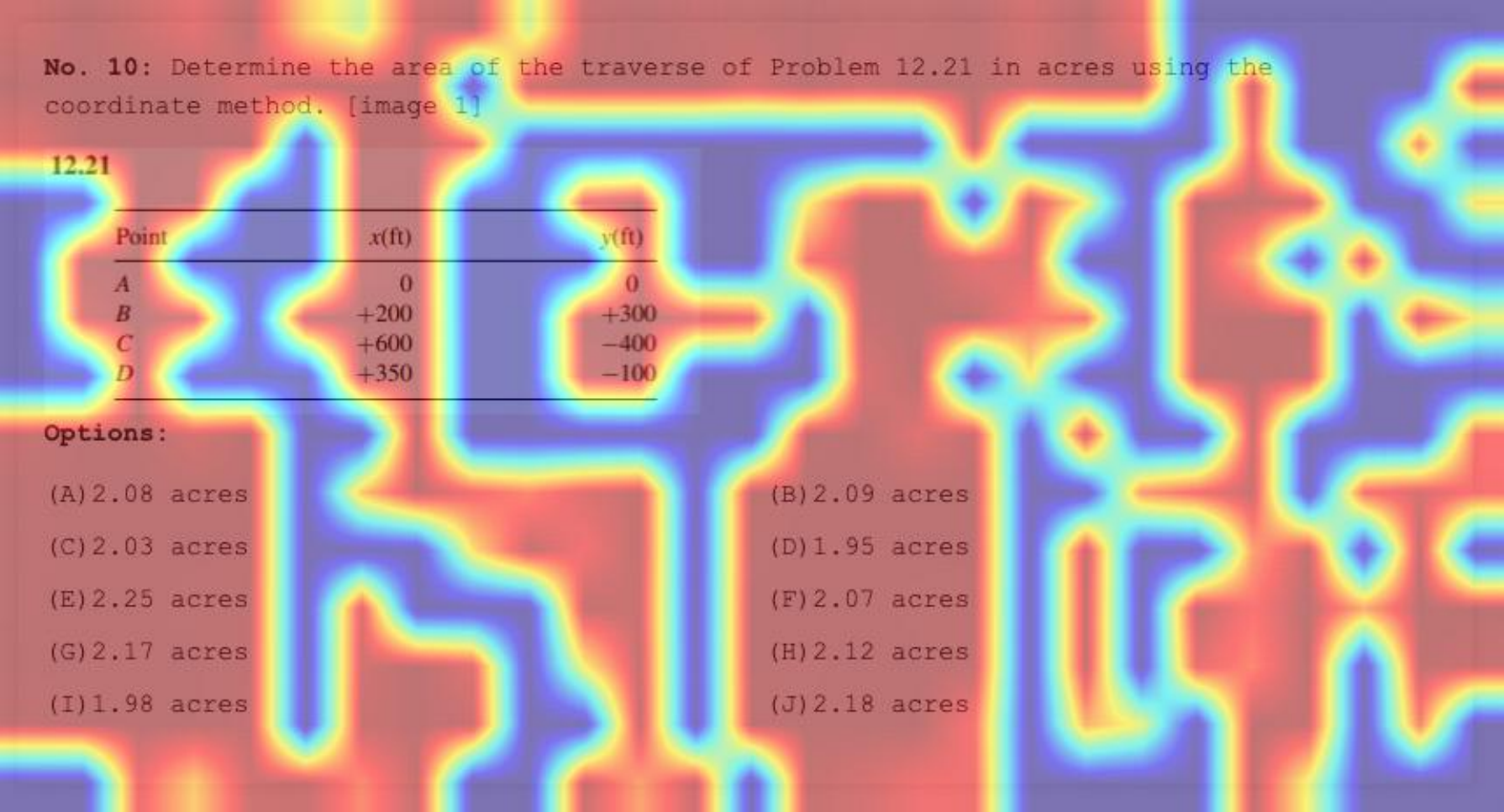}
        \caption{Generation step 15}
        \label{fig:visual_token_visualization_mmmu_pro_test_61:step_15}
    \end{subfigure}
    \hfill
    \begin{subfigure}[b]{0.48\textwidth}
        \centering
        \includegraphics[width=\textwidth]{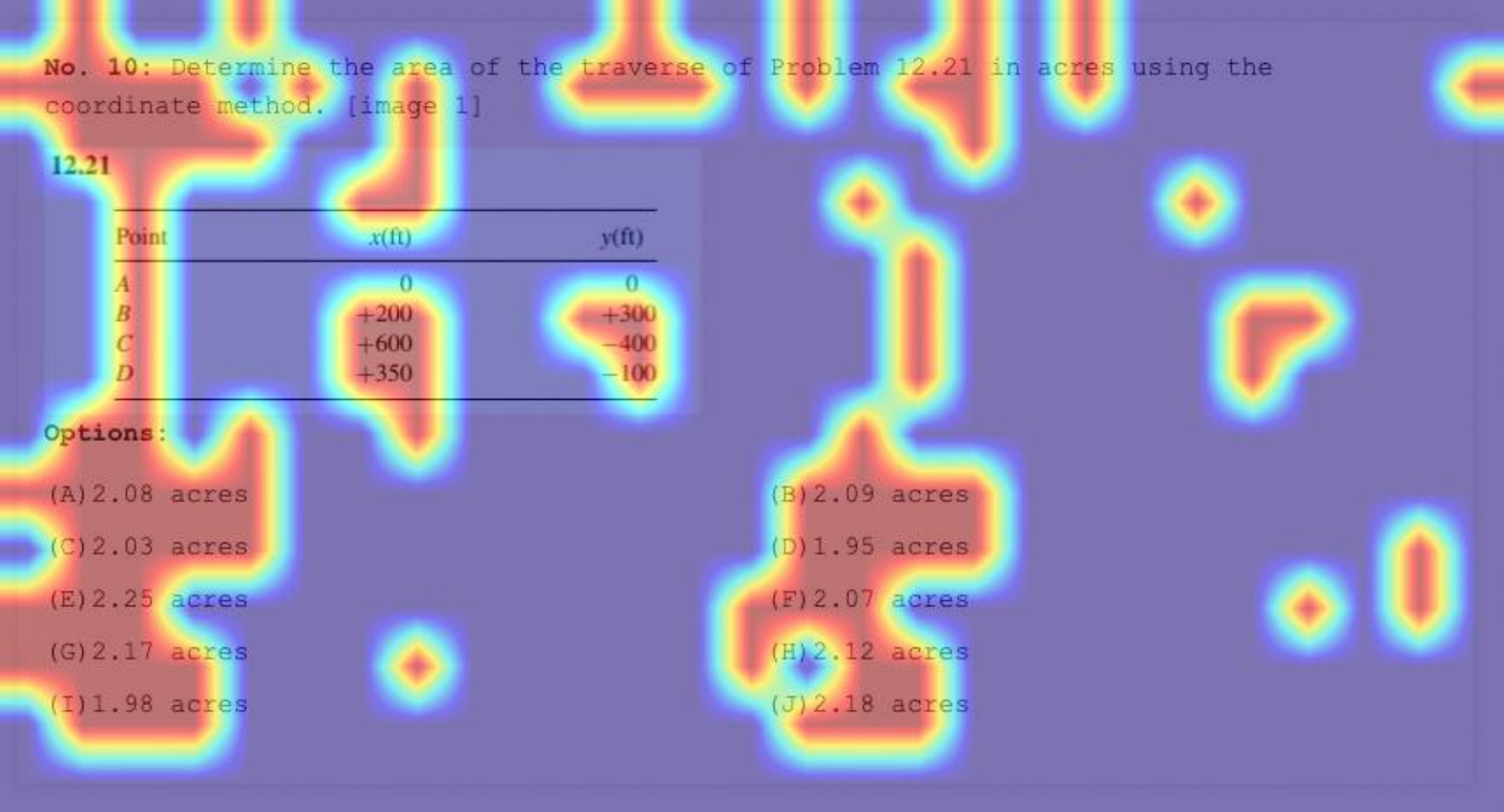}
        \caption{Last generation step}
        \label{fig:visual_token_visualization_mmmu_pro_test_61:last_step}
    \end{subfigure}

    \vspace{1em}
    
    \begin{subfigure}[b]{\textwidth}
        \centering
        \fbox{\parbox{0.95\textwidth}{\small
        let's solve this problem using the coordinate method... Area = (1/2)|sum from 1 to n of (xi*yi+1 - xi+1*yi)|... The problem gives points A, B, C, D. Let's confirm the order: A(0,0), B(200,300), C(600,-400), D(350,-100)... Area = (1/2)*| -180,000 | = 90,000 square feet... Now, we need to convert square feet to acres. We know that 1 acre = 43,560 square feet. So, Area in acres = 90,000 / 43,560 $\approx$ 2.066... So 2.066 is approximately 2.07, which is option F.
        }}
        \caption{Model reasoning}
        \label{fig:visual_token_visualization_mmmu_pro_test_61:reasoning}
    \end{subfigure}
    \caption{Visualization of retained visual tokens for a sample in the MMMU-Vision-Pro dataset (Layer 3, Head 3). The final generation step (b) demonstrates DBTrimKV isolating the specific textual coordinate data points (e.g., $A(0,0)$, $B(200,300)$) required by the model to formulate the area computation in its reasoning output (c). The correct answer is F.}
    \label{fig:visual_token_visualization_mmmu_pro_test_61}
\end{figure*}

\section{Limitations and Future works}

\label{appendix:section:limitations}

While our proposed DBTrimKV demonstrates strong performance, there are several limitations and avenues for future exploration:
\begin{itemize}[leftmargin=2em, itemsep=0pt, topsep=0pt, parsep=1pt, partopsep=1pt]
    \item In our current methodology, we only train the lightweight retention gates and freeze the base language model weights. While this approach is computationally efficient and preserves the fundamental capabilities of the base model, the impact of jointly fine-tuning the base LLM alongside the retention network remains unknown. Future work should investigate whether end-to-end joint training allows the model to inherently structure its internal representations to be more amenable to KV cache eviction.
    \item Our empirical evaluations were conducted on models up to the 8B parameter scale. We did not study the efficacy of globally calibrated retention gates on significantly larger, frontier-scale models. Establishing scaling laws for token retention and evaluating whether the mitigation of attention dilution behaves similarly at the 70B+ parameter scale is a critical next step.
    \item We observe that DBTrimKV can surpass full-cache inference performance at specific low-budget regimes by effectively suppressing distractor tokens. This indicates that token eviction can be utilized not just for compression, but as a mechanism to actively improve reasoning. Future work could explore using reinforcement learning as a post-training step, treating token retention as an action space to actively explore and optimize eviction policies for downstream performance.
\end{itemize}

\end{document}